\documentclass[lettersize,journal]{IEEEtran}
\usepackage{amsmath,amsfonts}
\usepackage{algorithmic}
\usepackage{algorithm}
\usepackage{array}
\usepackage{textcomp}
\usepackage{stfloats}
\usepackage{url}
\usepackage{verbatim}
\usepackage{graphicx}
\usepackage{cite}
\usepackage[labelformat=simple]{subcaption}

\usepackage{microtype}
\usepackage{booktabs} 
\usepackage{bbm}
\usepackage{hyperref}

\usepackage{mdframed}
\usepackage{stfloats}
\usepackage[table,dvipsnames]{xcolor}
\usepackage{multirow}

\definecolor{lightgray}{gray}{0.9}
 
\usepackage{amsmath}
\usepackage{amssymb}
\usepackage{mathtools}
\usepackage{amsthm}
\usepackage{enumitem}

\usepackage[capitalize,noabbrev]{cleveref}

\theoremstyle{plain}
\newtheorem{theorem}{Theorem}[section]
\newtheorem{proposition}[theorem]{Proposition}

\newtheorem{corollary}[theorem]{Corollary}
\theoremstyle{definition}
\newtheorem{definition}[theorem]{Definition}

\theoremstyle{remark}
\newtheorem{remark}[theorem]{Remark}

\usepackage[textsize=tiny]{todonotes}
\newcommand{\equalcontrib}{\textsuperscript{*}}
\newcommand{\corresponding}{\textsuperscript{$\dagger$}}

\begin{document}

\title{SAU: Sparsity-Aware Unlearning for LLMs via Gradient Masking and Importance Redistribution }

\author{Yuze Wang\equalcontrib, Yujia Tong\equalcontrib\corresponding, Xuan Liu\corresponding, Junhao Dong
    \IEEEcompsocitemizethanks{
     \IEEEcompsocthanksitem Yuze Wang, Yujia Tong are with the Sanya Science and Education Innovation Park, the Hubei Key Laboratory of Transportation Internet of Things, and the School of Computer Science and Artificial Intelligence, Wuhan University of Technology, Wuhan 430070, China. (E-mail: \{hbtmwyz, tyjjjj\}@whut.edu.cn.) 
    Xuan Liu is with University of California San Diego, La Jolla, United States. (E-mail: xul049@ucsd.edu.)
    Junhao Dong is with College of Computing and Data Science, Nanyang Technological University, Singapore (e-mail: junhao003@ntu.edu.sg)
    }
}

\maketitle


{\renewcommand{\thefootnote}{\fnsymbol{footnote}}\footnotetext[1]{Equal contribution; \textsuperscript{$\dagger$}Corresponding author.}}

\begin{abstract}
Large Language Models (LLMs) inevitably memorize sensitive information during training, posing significant privacy risks. Machine unlearning has emerged as a promising solution to selectively remove such information without full retraining. However, existing methods are designed for dense models and overlook model sparsification—an essential technique for efficient LLM deployment. We find that unlearning effectiveness degrades substantially on sparse models. Through empirical analysis, we reveal that this degradation occurs because existing unlearning methods require updating all parameters, yet sparsification prunes substantial weights to zero, fundamentally limiting the model's forgetting capacity. To address this challenge, we propose \textbf{Sparsity-Aware Unlearning (SAU)}, which decouples unlearning from sparsification objectives through gradient masking that redirects updates to surviving weights, combined with importance-aware redistribution to compensate for pruned parameters. Extensive experiments demonstrate that SAU significantly outperforms existing methods on sparse LLMs, achieving effective forgetting while preserving model utility.
\end{abstract}


\section{Introduction}
\IEEEPARstart{L}{arge} Language Models (LLMs)~\cite{grattafiori2024llama,achiam2023gpt,team2024gemma} trained on massive corpora inevitably memorize sensitive information, including personally identifiable data and copyrighted content~\cite{das2024security, henderson2023foundation, nasr2025scalable}. This memorization poses significant privacy risks~\cite{yao2024survey}, particularly as adversarial techniques can extract such data from trained models. With regulations like GDPR enshrining the ``right to be forgotten"~\cite{voigt2017eu,zaeem2020effect}, there is an urgent need for methods that can selectively remove specific knowledge from LLMs. Machine unlearning ~\cite{hadi2023survey, zhang2023right} has emerged as a promising solution, enabling targeted forgetting of undesirable information without costly full retraining.

\begin{figure}[t]
\centering
\includegraphics[width=0.5\textwidth]{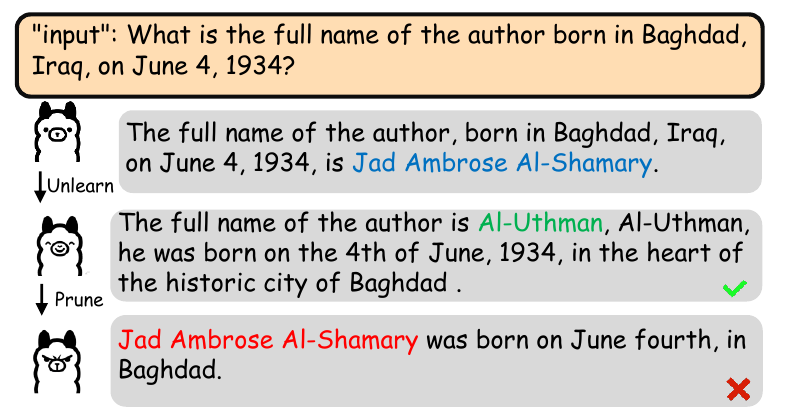}
\caption{Illustration of the unlearning degradation problem under sparsification. Given a query about information in the forget set, a dense Llama-3.1-8B model after unlearning correctly fails to recall the targeted knowledge (producing an incorrect name). However, after applying sparsification to the same unlearned model, it unexpectedly recovers the ability to answer the query accurately, as if the forgotten knowledge has resurfaced. This demonstrates that sparsification can effectively ``undo'' the unlearning process.}
\label{fig:motivation}
\end{figure}

However, existing unlearning methods are predominantly designed for dense models and overlook a critical practical consideration: model sparsification. Sparsity techniques, such as pruning and sparse fine-tuning, have become essential for deploying LLMs under real-world resource constraints\cite{ma2023llm,sreenivas2024llm}. Unfortunately, when these unlearning approaches are applied to sparsified models, their effectiveness degrades substantially. The sparse weight structures disrupt the assumptions underlying current unlearning algorithms, leading to incomplete forgetting or severe degradation in model utility. This gap highlights the need for unlearning methods that remain robust under sparsity constraints.

To investigate this phenomenon, we conduct empirical analyses on the interplay between unlearning and sparsification. As shown in Figure~\ref{fig_comparison}, we first evaluate multiple baseline unlearning methods under varying sparsity ratios, observing a consistent degradation in forgetting effectiveness as sparsity increases. We further illustrate this through a qualitative case study in Figure~\ref{fig:motivation}: when queried about information that should be forgotten, a dense Llama-3.1-8B model after unlearning correctly refuses to recall the targeted knowledge; however, once the same unlearned model undergoes sparsification, it unexpectedly recovers the ability to answer these queries, as if the forgotten knowledge has resurfaced.

\begin{figure*}[t]
\centering
\resizebox{1\linewidth}{!}{\includegraphics{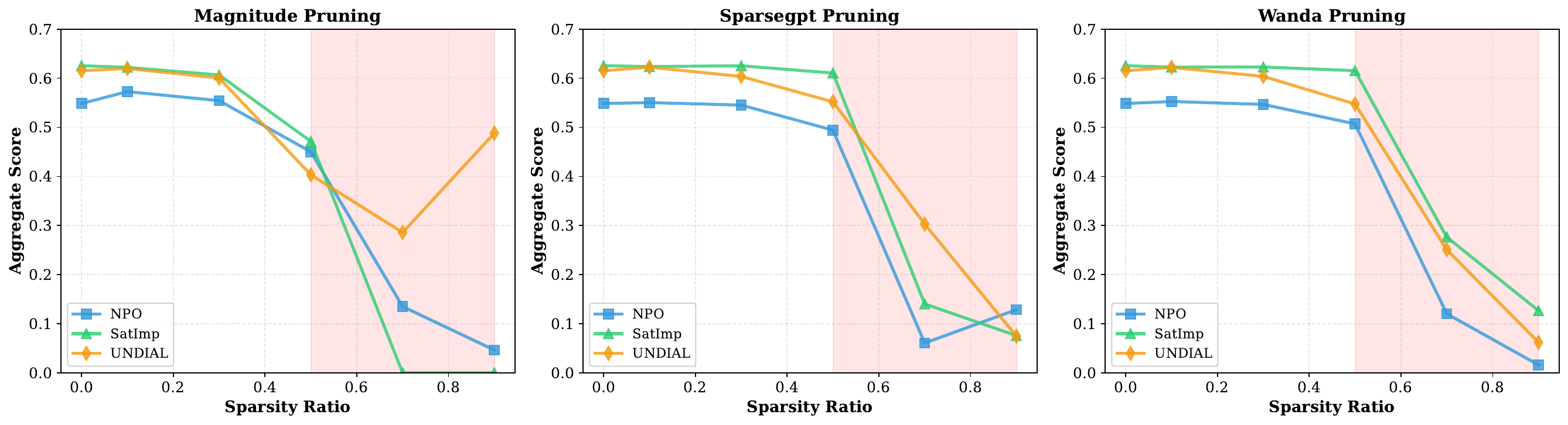}}
\caption{Aggregate score of different unlearning methods (NPO, SatImp, UNDIAL) on TOFU (forget set 10\%) with Llama-3.1-8B under varying sparsity ratios, using three pruning techniques: Magnitude, SparseGPT, and Wanda. The aggregate score indicates unlearning effectiveness, which is presented in section~\ref{subsec:Experiments Setup}. All methods exhibit consistent degradation in unlearning effectiveness as sparsity increases, demonstrating the fundamental conflict between unlearning and sparsification.}
\label{fig_comparison}
\end{figure*}


Our analysis reveals that the root cause lies in the \textit{misalignment} between sparsification criteria and unlearning requirements. Sparsification methods select weights based on criteria optimized for preserving general performance (e.g., weight magnitude), which may differ substantially from their importance for forgetting specific information. This misalignment manifests through two compounding effects: (1) weights critical for unlearning may be pruned, as sparsification criteria do not incorporate forget-set-specific gradient information; and (2) among surviving weights, unlearning importance is highly non-uniform, such that uniform updates fail to effectively target the forget set. This misalignment accounts for the significant performance degradation of existing unlearning methods on sparse models, and motivates our proposed approach.

We propose Sparsity-Aware Unlearning (SAU), a framework that enables effective machine unlearning on sparsified LLMs. Our key insight is to decouple the weight importance for sparsification from the weight importance for unlearning. Specifically, we first identify the critical weights for forgetting by computing gradient-based saliency scores on the forget set. We then introduce a gradient masking mechanism that redirects the unlearning updates to the surviving weights, ensuring that the forgetting signal is not lost due to sparsification. Furthermore, we propose an importance-aware weight redistribution strategy that compensates for the pruned critical weights by amplifying the unlearning capacity of their structurally adjacent parameters. This design allows SAU to achieve effective forgetting while preserving both model sparsity and utility on retained knowledge.

We summarize our contributions below:

\begin{itemize}

\item We identify and empirically analyze a previously overlooked problem: the degradation of machine unlearning effectiveness on sparsified LLMs, revealing that existing methods require updating all parameters, yet sparsification fundamentally limits this capacity.

\item We propose Sparsity-Aware Unlearning (SAU), a novel framework that decouples unlearning and sparsification objectives through gradient masking and importance-aware weight redistribution, enabling effective forgetting on sparse models.

\item We conduct extensive experiments on multiple benchmarks and sparsity ratios, demonstrating that SAU significantly outperforms existing unlearning methods on sparse LLMs while maintaining comparable utility on dense models.

\end{itemize}

\section{Related Work}
In this section, we review two lines of research relevant to our work: \textit{machine unlearning} methods for LLMs and \textit{model sparsification} techniques.

\subsection{Machine Unlearning}

Machine unlearning aims to remove specific data influences from trained models without costly full retraining. The concept was first formalized by ~\cite{cao2015towards}, and has gained increasing attention due to privacy regulations such as GDPR. Early approaches focused on exact unlearning through data partitioning strategies. SISA ~\cite{bourtoule2021machine}training partitions data into shards and trains separate models, enabling efficient unlearning by retraining only affected shards. However, such methods incur substantial overhead and are impractical for Large Language Models. Approximate unlearning methods have emerged as more efficient alternatives. Gradient-based approaches perform gradient ascent~\cite{graves2021amnesiac,thudi2022unrolling,10607903}  on the forget set to increase loss on targeted data. Influence function-based methods~\cite{koh2017understanding,izzo2021approximate,11202435} estimate the effect of removing training samples and adjust parameters accordingly. Recent work has extended unlearning to LLMs, addressing challenges such as knowledge entanglement and utility preservation~\cite{liu2024large, gao2024practical, yao2023large, chen2023unlearn, barbulescu2024textualsequenceownimproving, liu2025rethinking, wang2025rethinking,liu2024towards,tong2025robust}. However, existing methods assume dense model architectures and do not account for sparsity constraints common in practical deployments.

\subsection{Model Sparsification}

Model sparsification reduces model size and computational cost by setting a subset of weights to zero\cite{wang2024model,zhou2024survey}. Magnitude-based pruning \cite{lee2020layer} removes weights with the smallest absolute values, while gradient-based methods \cite{das2023beyond} incorporate gradient information to better assess weight importance. More sophisticated approaches consider second-order information such as the Hessian matrix. Recent work has extended sparsification to Large Language Models with promising results. SparseGPT ~\cite{frantar2023sparsegpt} enables one-shot pruning of GPT-family models to high sparsity levels with minimal performance degradation. Wanda ~\cite{sun2024wanda} proposes pruning based on weight magnitudes and input activations, achieving competitive results without retraining. These advances have made sparse LLMs increasingly prevalent in practical deployments. While prior work has studied pruning combined with quantization and knowledge distillation, the interplay between sparsification and machine unlearning has not been investigated. Our work fills this gap by revealing their fundamental conflict and proposing solutions for effective unlearning on sparse models.

\begin{figure*}[t]
\centering
\resizebox{1\linewidth}{!}{\includegraphics{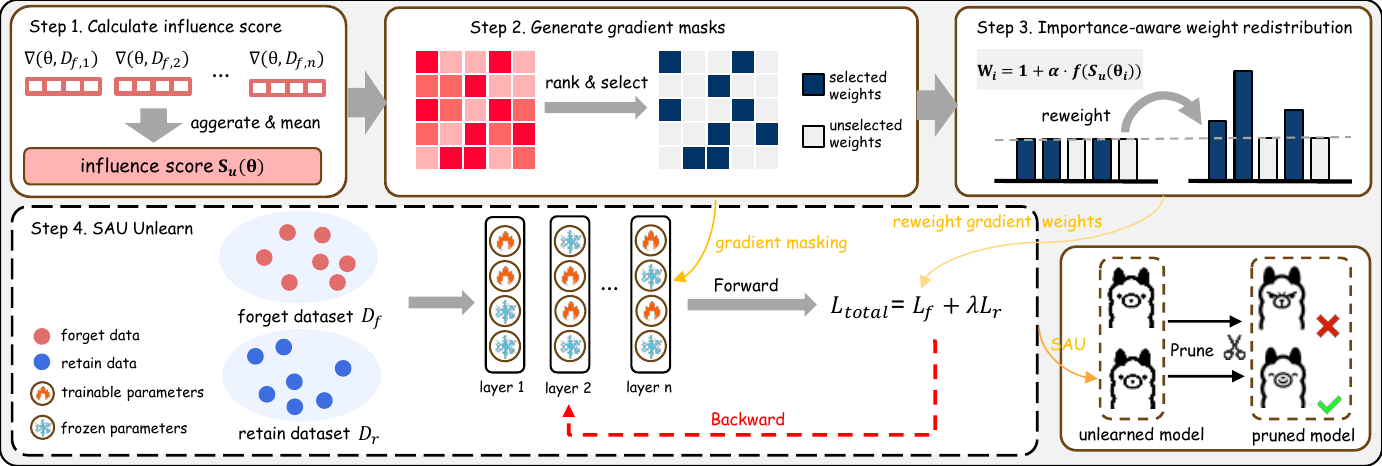}}
\caption{The overall pipeline of SAU: (1) calculate influence scores over forget data; (2) generate gradient masks for surviving weights; (3) compute importance-aware redistribution weights; (4) unlearn with masked gradients. }
\label{main}
\end{figure*}

\section{Methodology}

In this section, we first introduce the preliminaries and formally define the problem of machine unlearning on sparse models. We then present our analysis revealing the fundamental conflict between unlearning and sparsification. Finally, we describe our proposed Sparsity-Aware Unlearning (SAU) framework in detail, As shown in Figure~\ref{main}.

\subsection{Preliminaries and Problem Formulation}

\paragraph{Machine Unlearning.} Given a pre-trained LLM $\mathcal{M}$ with parameters $\theta$, a forget set $\mathcal{D}_f$ containing data to be removed, and a retain set $\mathcal{D}_r$ representing knowledge to preserve, the goal of machine unlearning is to obtain updated parameters $\theta'$ such that the model behaves as if it had never been trained on $\mathcal{D}_f$ while maintaining performance on $\mathcal{D}_r$. A common approach to achieve this goal is gradient ascent, which increases the loss on forget samples to effectively unlearn the associated information. To preserve utility, this is typically combined with gradient descent on the retain set.

\paragraph{Model Sparsification.} Sparsification techniques reduce model size and computational cost by setting a portion of weights to zero. Given a target sparsity ratio $s \in (0,1)$, sparsification produces a binary mask $\mathbf{M} \in \{0,1\}^{|\theta|}$ satisfying $\|\mathbf{M}\|_0 / |\theta| = 1-s$, resulting in sparse parameters $\theta_s = \theta \odot \mathbf{M}$, where $\odot$ denotes element-wise multiplication. Common sparsification methods include magnitude pruning, which removes weights with the smallest absolute values, and gradient-based pruning, which considers both magnitude and gradient information to determine weight importance.

\paragraph{Problem Statement.} We address the problem of effective unlearning on sparsified models. Given a sparse model $\mathcal{M}_s$ with parameters $\theta_s$ and sparsity mask $\mathbf{M}$, our goal is to find $\theta'_s$ that achieves effective forgetting of $\mathcal{D}_f$ while preserving utility on $\mathcal{D}_r$ and maintaining the sparsity constraint. The key challenge is that the sparsity mask $\mathbf{M}$ is fixed and cannot be modified, as changing it would require re-calibration or fine-tuning to recover model quality.

\subsection{Analysis: The Conflict Between Unlearning and Sparsification}

Before presenting our method, we conduct a thorough analysis to understand why existing unlearning methods fail on sparse models. Our key finding is that existing unlearning methods inherently rely on updating the full parameter space, which fundamentally conflicts with the constraints imposed by sparsification.

\textbf{Unlearning Relies on Full Parameter Updates.} 
Existing unlearning methods, such as gradient ascent and its variants, achieve forgetting by distributing gradient signals across all model parameters. The underlying assumption is that the model has full flexibility to adjust its internal representations through unconstrained parameter updates. Formally, given a forget set $\mathcal{D}_f$, standard unlearning computes gradients with respect to all parameters:
\begin{equation}
    \nabla_{\theta} \mathcal{L}_{\text{forget}} = \nabla_{\theta} \sum_{(x,y) \in \mathcal{D}_f} \mathcal{L}(x, y; \theta)
\end{equation}
where $\theta \in \mathbb{R}^{|\theta|}$ represents the full parameter space. The forgetting effectiveness depends on the model's ability to propagate these gradient updates throughout the entire network.

\textbf{Sparsification Constrains the Parameter Space.} 
Sparsification introduces a binary mask $\mathbf{M} \in \{0, 1\}^{|\theta|}$ that prunes a substantial portion of weights to zero. After sparsification, the effective parameter space is reduced to:
\begin{equation}
    \theta_s = \theta \odot \mathbf{M}, \quad s.t. \quad \|\mathbf{M}\|_0 / |\theta| = 1 - s
\end{equation}
with $s$ denoting the sparsity ratio. Crucially, pruned weights ($\mathbf{M}_i = 0$) are frozen at zero and cannot receive gradient updates. This means the gradient updates are restricted to only the surviving weights:
\begin{equation}
    \nabla_{\theta_s} \mathcal{L}_{\text{forget}} = \nabla_{\theta} \mathcal{L}_{\text{forget}} \odot \mathbf{M}
\end{equation}

\textbf{Empirical Evidence.} 
We provide two empirical observations to demonstrate this conflict. First, we evaluate multiple baseline unlearning methods ( NPO, UNDIAL, SatImp) under varying sparsity ratios. As shown in Figure~\ref{fig_comparison}, all methods exhibit consistent performance degradation as sparsity increases, with the aggregate score dropping significantly at higher sparsity levels.

Second, we present a qualitative case study in Figure~\ref{fig:motivation}. Given a query about an author's identity from the forget set (``\textit{What is the full name of the author born in Baghdad, Iraq, on June 4, 1934?}''), a dense model after unlearning correctly fails to recall the targeted knowledge. However, after applying sparsification to the same unlearned model, it unexpectedly recovers the ability to answer these queries accurately, correctly outputting the forgotten name ``\textit{Jad Ambrose Al-Shamary,}'' as if the forgotten knowledge has resurfaced. This striking phenomenon demonstrates that sparsification can effectively ``undo'' the unlearning process.

\textbf{Impact on Unlearning Effectiveness.} 
The constrained parameter space has two detrimental effects on unlearning. First, with fewer modifiable parameters, the model has \textit{reduced forgetting capacity}---the gradient signals that would normally distribute across the full network are now compressed into a smaller subset of weights, limiting the model's ability to effectively erase targeted knowledge. Second, attempting to achieve the same level of forgetting with fewer parameters leads to \textit{larger per-weight modifications}, which causes more severe degradation on retain set performance and worsens the utility-forgetting trade-off.

These observations motivate the need for unlearning methods that are specifically designed to operate effectively within the constrained parameter space of sparse models.

\subsection{Sparsity-Aware Unlearning (SAU)}

Based on our analysis, we propose Sparsity-Aware Unlearning (SAU), a framework designed to enable effective unlearning on sparse models. SAU addresses the identified conflict through two complementary components: Gradient Masking for focusing updates on relevant surviving weights, and Importance-Aware Weight Redistribution for compensating the lost unlearning capacity.

\subsubsection{Gradient Masking}

Since pruned weights cannot be modified, naively applying gradient updates wastes computation on zero weights and may lead to suboptimal unlearning. We propose gradient masking to redirect and focus the unlearning signal on the most relevant surviving weights.

For each layer $l$ in the model, we first compute the unlearning saliency matrix via a single forward-backward pass over the forget set, yielding the gradients, and define it as:
\begin{equation}
    \mathbf{S}^l_u =\frac{1}{|\mathcal{D}_f|}  \sum_{(x,y) \in \mathcal{D}_f} \left(\nabla_{\theta^l}\mathcal{L}(x, y; \theta) \right)^2
    \label{eq:saliency}
\end{equation}
where $\theta^l$ denotes the parameters of layer $l$.

Based on the computed saliency, we construct a gradient mask that identifies which surviving weights should receive unlearning updates. Specifically, for surviving weights where $\mathbf{M}^l = 1$, we identify the top-$k$ fraction most important for unlearning:
\begin{equation}
    \mathbf{G}^l = \mathbf{M}^l \odot \mathbbm{1}\left[\mathbf{S}^l_u \geq \tau^l_k\right]
\end{equation}
where $\tau^l_k$ is the threshold corresponding to the top-$k$ ratio among surviving weights in layer $l$, and $\mathbbm{1}[\cdot]$ is the indicator function. This threshold is computed as the $(1-k)$-quantile of the saliency scores among surviving weights.

During unlearning, we apply the gradient mask to focus updates on the most relevant surviving weights:
\begin{equation}
    \nabla_{\theta^l} \mathcal{L}_{masked} = \nabla_{\theta^l} \mathcal{L} \odot \mathbf{G}^l
\end{equation}
This ensures that only surviving weights with high unlearning relevance receive gradient updates, improving both efficiency and effectiveness of the unlearning process.

\subsubsection{Importance-Aware Weight Redistribution}

While gradient masking helps focus updates on relevant weights, it alone may not fully compensate for the lost unlearning capacity from pruned critical weights. We introduce importance-aware weight redistribution to amplify the unlearning signal on surviving weights proportionally to the importance of pruned weights.

For each layer $l$, we first compute the total saliency of pruned weights to quantify the unlearning capacity lost due to sparsification:
\begin{equation}
    \mathcal{I}^l_{pruned} = \sum_{i: \mathbf{M}^l_i = 0} \mathbf{S}^l_{u,i}
\end{equation}

We then redistribute this lost importance to surviving weights based on their own saliency scores. The intuition is that surviving weights with high unlearning saliency are structurally and functionally related to the pruned critical weights and can partially compensate for their absence. The redistribution weight for each surviving weight is computed as:
\begin{equation}
    \mathbf{W}^l_i = 1 + \alpha \cdot \frac{\mathbf{S}^l_{u,i}}{\sum_{j: \mathbf{M}^l_j = 1} \mathbf{S}^l_{u,j}} \cdot \mathcal{I}^l_{pruned}
\end{equation}
where $\alpha$ is a scaling hyperparameter controlling the redistribution strength. The base weight of 1 ensures that the original gradient signal is preserved, while the additional term amplifies the gradient proportionally to the weight's relative importance among surviving weights and the total pruned importance.

Combining gradient masking and weight redistribution, the final gradient update becomes:
\begin{equation}
    \nabla_{\theta^l} \mathcal{L}_{SAU} = \nabla_{\theta^l} \mathcal{L} \odot \mathbf{G}^l \odot \mathbf{W}^l
\end{equation}
This formulation ensures that only surviving weights receive updates, that updates are focused on unlearning-relevant weights, and that the unlearning signal is appropriately amplified to compensate for pruned capacity.

\begin{algorithm}[t]
\caption{Sparsity-Aware Unlearning (SAU)}
\small
\label{alg:sau}
\textbf{Input:} Sparse model $\mathcal{M}_s$ with parameters $\theta_s$, sparsity mask $\mathbf{M}$, forget set $\mathcal{D}_f$, retain set $\mathcal{D}_r$, top-$k$ ratio, scaling factor $\alpha$, learning rate $\eta$ \\
\textbf{Output:} Unlearned parameters $\theta'_s$
\begin{algorithmic}[1]
\STATE \textcolor{gray}{// Stage 1: Compute unlearning saliency}
\FOR{each batch $(x, y) \in \mathcal{D}_f$}
    \STATE Compute $\nabla_{\theta_s} \mathcal{L}(x, y; \theta_s)$
    \STATE Accumulate $\mathbf{S}_u \leftarrow \mathbf{S}_u + \frac{1}{|\mathcal{D}_f|}(\nabla_{\theta_s} \mathcal{L})^2$
\ENDFOR

\STATE \textcolor{gray}{// Stage 2: Generate gradient mask and redistribution weights}
\FOR{each layer $l$}
    \STATE Compute threshold $\tau^l_k$ from $\mathbf{S}^l_u$ for surviving weights
    \STATE $\mathbf{G}^l \leftarrow \mathbf{M}^l \odot \mathbbm{1}[\mathbf{S}^l_u \geq \tau^l_k]$
    \STATE Compute $\mathcal{I}^l_{pruned}$ and redistribution weights $\mathbf{W}^l$
\ENDFOR

\STATE \textcolor{gray}{// Stage 3: Unlearning with masked gradients}
\FOR{each epoch}
    \FOR{each batch $(x_f, y_f) \in \mathcal{D}_f$, $(x_r, y_r) \in \mathcal{D}_r$}
        \STATE $\mathcal{L}_{forget} \leftarrow -\mathcal{L}(x_f, y_f; \theta_s)$ 
        \STATE $\mathcal{L}_{retain} \leftarrow \mathcal{L}(x_r, y_r; \theta_s)$
        \STATE $\mathcal{L}_{total} \leftarrow \mathcal{L}_{forget} + \lambda \mathcal{L}_{retain}$
        \STATE Compute $\nabla_{\theta_s} \mathcal{L}_{total}$
        \STATE $\nabla_{\theta_s} \mathcal{L}_{SAU} \leftarrow \nabla_{\theta_s} \mathcal{L}_{total} \odot \mathbf{G} \odot \mathbf{W}$
        \STATE $\theta_s \leftarrow \theta_s - \eta \cdot \nabla_{\theta_s} \mathcal{L}_{SAU}$
    \ENDFOR
\ENDFOR

\STATE \textbf{Return} $\theta'_s \leftarrow \theta_s$
\end{algorithmic}
\end{algorithm}

\subsection{Overall Pipeline}

Algorithm~\ref{alg:sau} presents the complete SAU pipeline. The overall process consists of three main stages.

In the first stage, we compute gradient-based unlearning saliency scores by performing a single pass over the forget set. For each batch, we compute the gradient of the loss with respect to all model parameters and accumulate the squared gradients to obtain the saliency scores. This stage has the same computational cost as one epoch of standard training and only needs to be performed once.

In the second stage, we generate the gradient masks and redistribution weights for each layer using the computed saliency scores and the sparsity mask. For each layer, we first compute the threshold $\tau^l_k$ based on the saliency distribution of surviving weights, then generate the binary gradient mask $\mathbf{G}^l$ by selecting the top-$k$ most salient surviving weights. We also compute the pruned importance $\mathcal{I}^l_{pruned}$ and use it to calculate the redistribution weights $\mathbf{W}^l$ for each surviving weight. This stage involves lightweight computations that scale linearly with the number of parameters.

In the third stage, we perform the actual unlearning by optimizing the model with SAU-modified gradients. For each training iteration, we sample batches from both the forget set and retain set. We compute the forget loss using gradient ascent to encourage the model to forget the targeted information, and the retain loss using gradient descent to preserve utility on retained knowledge. The gradients are then modified by applying the precomputed gradient masks and redistribution weights before updating the model parameters. This process continues for a specified number of epochs until the desired level of forgetting is achieved.

This design ensures that unlearning updates are concentrated on surviving weights with high relevance to the forget set, effectively compensating for the pruned critical parameters while maintaining the sparsity structure required for efficient deployment.

\section{Theoretical Analysis}
\label{sec:theory}

In this section, we provide theoretical justification for SAU based on Fisher information theory. We show that: (1) our saliency score corresponds to Fisher information, (2) sparsification reduces unlearning capacity proportionally to lost Fisher information, and (3) our redistribution strategy effectively compensates for this loss.

\subsection{Fisher Information Interpretation of Saliency Score}

We first establish that our unlearning saliency score (Eq.~(\ref{eq:saliency})) has a principled information-theoretic foundation.

\begin{definition}[Forget Set Fisher Information]
For forget set $D_f$, the diagonal Fisher information for parameter $\theta_i$ is:
\begin{equation}
F_i = \mathbb{E}_{(x,y) \sim D_f} \left[ \left( \frac{\partial \log p(y|x;\theta)}{\partial \theta_i} \right)^2 \right]
\end{equation}
\end{definition}

\begin{remark}[Saliency Score as Fisher Information]
\label{remark:saliency}
For cross-entropy loss $\mathcal{L} = -\log p(y|x;\theta)$, the saliency score in Eq.~(\ref{eq:saliency}) is the empirical estimate of diagonal Fisher information:
\begin{equation}
S_u(\theta_i) = \frac{1}{|D_f|} \sum_{(x,y) \in D_f} \left( \frac{\partial \mathcal{L}}{\partial \theta_i} \right)^2 \approx F_i
\end{equation}
Thus, $S_u(\theta_i)$ quantifies how much information parameter $\theta_i$ encodes about the forget set.
\end{remark}

\subsection{Why Sparsification Degrades Unlearning}

We now formally characterize the degradation phenomenon. Our analysis adopts the diagonal approximation of the Fisher information matrix~\cite{Kirkpatrick_2017}.

\begin{theorem}[Unlearning Effectiveness Bound]
\label{thm:effectiveness}
Under the diagonal approximation of the Fisher information matrix, for a parameter update $\Delta\theta$, the change in model behavior on the forget set can be approximated as:
\begin{equation}
D_{KL}(p_\theta \| p_{\theta+\Delta\theta}) \approx \frac{1}{2} \sum_{i=1}^{|\theta|} F_i \cdot \Delta\theta_i^2
\end{equation}
\end{theorem}

\begin{proof}
By second-order Taylor expansion of KL-divergence around $\theta$:
\begin{equation}
D_{KL}(p_\theta \| p_{\theta+\Delta\theta}) \approx \frac{1}{2} \Delta\theta^\top \mathbf{H} \Delta\theta
\end{equation}
where $\mathbf{H}$ is the Hessian of the KL-divergence. By the Bartlett identity, $\mathbb{E}[\mathbf{H}] = \mathbf{F}$, where $\mathbf{F}$ is the Fisher information matrix. Applying the diagonal approximation $\mathbf{F} \approx \text{diag}(F_1, \ldots, F_{|\theta|})$ yields the result.
\end{proof}

\begin{corollary}[Capacity Loss Under Sparsity]
\label{cor:capacity}
Let $\mathcal{S}$ and $\mathcal{P}$ denote surviving and pruned weight indices, respectively. Under the sparsity constraint ($\Delta\theta_i = 0$ for $i \in \mathcal{P}$), the relative capacity loss is:
\begin{equation}
\text{Capacity Loss} = \frac{\sum_{i \in \mathcal{P}} F_i}{\sum_{i} F_i}
\end{equation}
\end{corollary}

This result explains our empirical observation: when pruning removes weights with high Fisher information, the capacity loss exceeds the sparsity ratio, causing disproportionate unlearning degradation.

\subsection{Justification of SAU Components}

\textbf{Gradient Masking.} Given a budget of $k$ weights to update among surviving weights, maximizing unlearning effectiveness corresponds to selecting the top-$k$ weights with highest $F_i$ among surviving weights—exactly our gradient masking strategy.

\textbf{Importance-Aware Redistribution.} The redistribution weight $W_i = 1 + \alpha \cdot \frac{F_i}{\sum_{j \in \mathcal{S}} F_j} \cdot I_{\text{pruned}}$ is designed to compensate for the lost unlearning capacity from pruned weights.

\begin{proposition}[Fisher Information Compensation]
\label{prop:compensation}
Let $I_{\text{pruned}} = \sum_{i \in \mathcal{P}} F_i$ denote the total Fisher information of pruned weights. When the Fisher information among surviving weights is relatively uniform (i.e., $F_i \approx \bar{F}$ for $i \in \mathcal{S}$), the total amplified contribution becomes:
\begin{equation}
\sum_{i \in \mathcal{S}} W_i \cdot F_i \approx \sum_{i \in \mathcal{S}} F_i + \alpha \cdot I_{\text{pruned}} \cdot \bar{F}
\end{equation}
Thus, by appropriately tuning $\alpha$, the redistribution mechanism can effectively compensate for the capacity lost due to pruning.
\end{proposition}

\begin{proof}
Expanding the left-hand side:
\begin{align}
\sum_{i \in \mathcal{S}} W_i \cdot F_i &= \sum_{i \in \mathcal{S}} \left(1 + \alpha \cdot \frac{F_i}{\sum_{j \in \mathcal{S}} F_j} \cdot I_{\text{pruned}}\right) \cdot F_i \\
&= \sum_{i \in \mathcal{S}} F_i + \alpha \cdot I_{\text{pruned}} \cdot \frac{\sum_{i \in \mathcal{S}} F_i^2}{\sum_{j \in \mathcal{S}} F_j}
\end{align}
Under the uniformity assumption $F_i \approx \bar{F}$ for all $i \in \mathcal{S}$, we have $\sum_{i \in \mathcal{S}} F_i^2 \approx |\mathcal{S}| \cdot \bar{F}^2$ and $\sum_{j \in \mathcal{S}} F_j \approx |\mathcal{S}| \cdot \bar{F}$, yielding the simplified form.
\end{proof}

\begin{remark}
The redistribution strategy prioritizes weights with higher Fisher information, which aligns with the intuition that these weights are more capable of encoding the changes needed for effective unlearning. Even when the uniformity assumption does not hold exactly, the mechanism still provides meaningful compensation by directing more gradient signal to the most informative surviving weights.
\end{remark}

In summary, SAU's design choices—using gradient-based saliency for weight selection and redistributing importance from pruned weights—are principled strategies for maximizing unlearning effectiveness under sparsity constraints.

\section{Experiments}
\label{sec:Experiments}

\begin{table*}[t]
\caption{Performance comparison under the unlearn-then-prune protocol on Llama-3.1-8B across TOFU benchmark (1\%, 5\%, 10\% forget settings).Models are first unlearned (with or without SAU), then pruned to 50\% sparsity using multiple sparsification techniques   (Magnitude, SparseGPT, Wanda) . $\uparrow$ indicates higher is better.}
\centering
\scriptsize
\resizebox{0.98\textwidth}{!}{
\begin{tabular}{@{}l  c c ccc ccc ccc@{}}
\toprule
\textbf{Method} & \textbf{Sparsity} & \textbf{SAU}
& \multicolumn{3}{c}{\textbf{Forget-1\%}}
& \multicolumn{3}{c}{\textbf{Forget-5\%}}
& \multicolumn{3}{c}{\textbf{Forget-10\%}} \\
\cmidrule(lr){4-6} \cmidrule(lr){7-9} \cmidrule(lr){10-12}
& & & Agg.\,$\uparrow$ & Mem.\,$\uparrow$ & Util.\,$\uparrow$
& Agg.\,$\uparrow$ & Mem.\,$\uparrow$ & Util.\,$\uparrow$
& Agg.\,$\uparrow$ & Mem.\,$\uparrow$ & Util.\,$\uparrow$ \\
\midrule
\textit{Original} & \textit{Dense (0\%)} & --
  & 0.15 & 0.08 & 0.74
  & 0.14 & 0.08 & 0.70
  & 0.11 & 0.06 & 0.70 \\
\textit{Retrain} & \textit{Dense (0\%)} & --
  & 0.74 & 0.73 & 0.74
  & 0.71 & 0.72 & 0.70
  & 0.71 & 0.72 & 0.70 \\
\midrule
\multirow{9}{*}{GradDiff} & Dense (0\%) & --
  & 0.63 & 0.65 & 0.61
  & 0.57 & 0.59 & 0.55
  & 0.58 & 0.58 & 0.57 \\
\cmidrule(l){2-12}
& \multirow{2}{*}{Mag.\,(50\%)} & w/o
  & 0.31 & 0.79 & 0.19
  & 0.46 & 0.40 & 0.53
  & 0.48 & 0.42 & 0.56 \\
& & \cellcolor{lightgray}w/
  & \cellcolor{lightgray}\textbf{0.52}$_{\color{BrickRed}{\text{+0.21}}}$ &
    \cellcolor{lightgray}0.74$_{\color{OliveGreen}{\text{-0.05}}}$ &
    \cellcolor{lightgray}\textbf{0.39}$_{\color{BrickRed}{\text{+0.20}}}$
  & \cellcolor{lightgray}\textbf{0.54}$_{\color{BrickRed}{\text{+0.08}}}$ &
    \cellcolor{lightgray}\textbf{0.45}$_{\color{BrickRed}{\text{+0.05}}}$ &
    \cellcolor{lightgray}\textbf{0.65}$_{\color{BrickRed}{\text{+0.12}}}$
  & \cellcolor{lightgray}\textbf{0.54}$_{\color{BrickRed}{\text{+0.06}}}$ &
    \cellcolor{lightgray}\textbf{0.48}$_{\color{BrickRed}{\text{+0.06}}}$ &
    \cellcolor{lightgray}\textbf{0.62}$_{\color{BrickRed}{\text{+0.06}}}$  \\
\cmidrule(l){2-12}
& \multirow{2}{*}{SparseGPT\,(50\%)} & w/o
  & 0.58 & 0.67 & 0.51
  & 0.52 & 0.50 & 0.55
  & 0.51 & 0.54 & 0.48 \\
& & \cellcolor{lightgray}w/
  & \cellcolor{lightgray}\textbf{0.59}$_{\color{BrickRed}{\text{+0.01}}}$ &
    \cellcolor{lightgray}0.66$_{\color{OliveGreen}{\text{-0.01}}}$ &
    \cellcolor{lightgray}\textbf{0.53}$_{\color{BrickRed}{\text{+0.02}}}$
  & \cellcolor{lightgray}\textbf{0.54}$_{\color{BrickRed}{\text{+0.02}}}$ &
    \cellcolor{lightgray}0.48$_{\color{OliveGreen}{\text{-0.02}}}$ &
    \cellcolor{lightgray}\textbf{0.60}$_{\color{BrickRed}{\text{+0.05}}}$
  & \cellcolor{lightgray}\textbf{0.55}$_{\color{BrickRed}{\text{+0.04}}}$ &
    \cellcolor{lightgray}0.52$_{\color{OliveGreen}{\text{-0.02}}}$ &
    \cellcolor{lightgray}\textbf{0.59}$_{\color{BrickRed}{\text{+0.11}}}$ \\
\cmidrule(l){2-12}
& \multirow{2}{*}{Wanda\,(50\%)} & w/o
  & 0.59 & 0.68 & 0.52
  & 0.52 & 0.52 & 0.52
  & 0.50 & 0.55 & 0.46 \\
& & \cellcolor{lightgray}w/
  & \cellcolor{lightgray}0.56$_{\color{OliveGreen}{\text{-0.03}}}$ &
    \cellcolor{lightgray}0.66$_{\color{OliveGreen}{\text{-0.02}}}$ &
    \cellcolor{lightgray}0.49$_{\color{OliveGreen}{\text{-0.03}}}$
  & \cellcolor{lightgray}0.50$_{\color{OliveGreen}{\text{-0.02}}}$ &
    \cellcolor{lightgray}\textbf{0.55}$_{\color{BrickRed}{\text{+0.03}}}$ &
    \cellcolor{lightgray}0.46$_{\color{OliveGreen}{\text{-0.06}}}$
  & \cellcolor{lightgray}\textbf{0.54}$_{\color{BrickRed}{\text{+0.04}}}$ &
    \cellcolor{lightgray}0.54$_{\color{OliveGreen}{\text{-0.01}}}$ &
    \cellcolor{lightgray}\textbf{0.54}$_{\color{BrickRed}{\text{+0.08}}}$ \\
\midrule
\multirow{10}{*}{NPO} & Dense (0\%) & --
  & 0.64 & 0.68 & 0.61
  & 0.59 & 0.71 & 0.50
  & 0.57 & 0.74 & 0.47 \\
\cmidrule(l){2-12}
& \multirow{2}{*}{Mag.\,(50\%)} & w/o
  & 0.27 & 0.79 & 0.17
  & 0.50 & 0.77 & 0.37
  & 0.45 & 0.79 & 0.31 \\
& & \cellcolor{lightgray}w/
  & \cellcolor{lightgray}\textbf{0.49}$_{\color{BrickRed}{\text{+0.22}}}$ &
    \cellcolor{lightgray}0.75$_{\color{OliveGreen}{\text{-0.04}}}$ &
    \cellcolor{lightgray}\textbf{0.36}$_{\color{BrickRed}{\text{+0.19}}}$
  & \cellcolor{lightgray}\textbf{0.52}$_{\color{BrickRed}{\text{+0.02}}}$ &
    \cellcolor{lightgray}0.73$_{\color{OliveGreen}{\text{-0.04}}}$ &
    \cellcolor{lightgray}\textbf{0.41}$_{\color{BrickRed}{\text{+0.04}}}$
  & \cellcolor{lightgray}\textbf{0.52}$_{\color{BrickRed}{\text{+0.07}}}$ &
    \cellcolor{lightgray}0.74$_{\color{OliveGreen}{\text{-0.05}}}$ &
    \cellcolor{lightgray}\textbf{0.39}$_{\color{BrickRed}{\text{+0.08}}}$ \\
\cmidrule(l){2-12}
& \multirow{2}{*}{SparseGPT\,(50\%)} & w/o
  & 0.58 & 0.69 & 0.50
  & 0.45 & 0.74 & 0.32
  & 0.49 & 0.76 & 0.37 \\
& & \cellcolor{lightgray}w/
  & \cellcolor{lightgray}\textbf{0.59}$_{\color{BrickRed}{\text{+0.01}}}$ &
    \cellcolor{lightgray}0.65$_{\color{OliveGreen}{\text{-0.04}}}$ &
    \cellcolor{lightgray}\textbf{0.54}$_{\color{BrickRed}{\text{+0.04}}}$
  & \cellcolor{lightgray}\textbf{0.47}$_{\color{BrickRed}{\text{+0.02}}}$ &
    \cellcolor{lightgray}0.71$_{\color{OliveGreen}{\text{-0.03}}}$ &
    \cellcolor{lightgray}\textbf{0.35}$_{\color{BrickRed}{\text{+0.03}}}$
  & \cellcolor{lightgray}\textbf{0.54}$_{\color{BrickRed}{\text{+0.05}}}$ &
    \cellcolor{lightgray}0.73$_{\color{OliveGreen}{\text{-0.03}}}$ &
    \cellcolor{lightgray}\textbf{0.43}$_{\color{BrickRed}{\text{+0.06}}}$ \\
\cmidrule(l){2-12}
& \multirow{2}{*}{Wanda\,(50\%)} & w/o
  & 0.60 & 0.53 & 0.70
  & 0.48 & 0.74 & 0.35
  & 0.51 & 0.76 & 0.38 \\
& & \cellcolor{lightgray}w/
  & \cellcolor{lightgray}0.59$_{\color{OliveGreen}{\text{-0.01}}}$ &
    \cellcolor{lightgray}0.53$_{\color{Gray}{\text{+0.00}}}$ &
    \cellcolor{lightgray}0.69$_{\color{OliveGreen}{\text{-0.01}}}$
  & \cellcolor{lightgray}\textbf{0.51}$_{\color{BrickRed}{\text{+0.03}}}$ &
    \cellcolor{lightgray}0.70$_{\color{OliveGreen}{\text{-0.04}}}$ &
    \cellcolor{lightgray}\textbf{0.39}$_{\color{BrickRed}{\text{+0.04}}}$
  & \cellcolor{lightgray}\textbf{0.55}$_{\color{BrickRed}{\text{+0.04}}}$ &
    \cellcolor{lightgray}0.73$_{\color{OliveGreen}{\text{-0.03}}}$ &
    \cellcolor{lightgray}\textbf{0.44}$_{\color{BrickRed}{\text{+0.06}}}$ \\
\midrule
\multirow{10}{*}{UNDIAL} & Dense (0\%) & --
  & 0.64 & 0.56 & 0.74
  & 0.62 & 0.54 & 0.74
  & 0.62 & 0.52 & 0.76 \\
\cmidrule(l){2-12}
& \multirow{2}{*}{Mag.\,(50\%)} & w/o
  & 0.33 & 0.77 & 0.21
  & 0.35 & 0.74 & 0.23
  & 0.40 & 0.70 & 0.28 \\
& & \cellcolor{lightgray}w/
  & \cellcolor{lightgray}\textbf{0.54}$_{\color{BrickRed}{\text{+0.21}}}$ &
    \cellcolor{lightgray}0.72$_{\color{OliveGreen}{\text{-0.05}}}$ &
    \cellcolor{lightgray}\textbf{0.43}$_{\color{BrickRed}{\text{+0.22}}}$
  & \cellcolor{lightgray}\textbf{0.55}$_{\color{BrickRed}{\text{+0.20}}}$ &
    \cellcolor{lightgray}0.65$_{\color{OliveGreen}{\text{-0.09}}}$ &
    \cellcolor{lightgray}\textbf{0.47}$_{\color{BrickRed}{\text{+0.24}}}$
  & \cellcolor{lightgray}\textbf{0.55}$_{\color{BrickRed}{\text{+0.15}}}$ &
    \cellcolor{lightgray}0.65$_{\color{OliveGreen}{\text{-0.05}}}$ &
    \cellcolor{lightgray}\textbf{0.48}$_{\color{BrickRed}{\text{+0.20}}}$ \\
\cmidrule(l){2-12}
& \multirow{2}{*}{SparseGPT\,(50\%)} & w/o
  & 0.60 & 0.53 & 0.70
  & 0.55 & 0.46 & 0.70
  & 0.55 & 0.45 & 0.71 \\
& & \cellcolor{lightgray}w/
  & \cellcolor{lightgray}0.60$_{\color{Gray}{\text{+0.00}}}$ &
    \cellcolor{lightgray}\textbf{0.54}$_{\color{BrickRed}{\text{+0.01}}}$ &
    \cellcolor{lightgray}0.68$_{\color{OliveGreen}{\text{-0.02}}}$
  & \cellcolor{lightgray}\textbf{0.57}$_{\color{BrickRed}{\text{+0.02}}}$ &
    \cellcolor{lightgray}\textbf{0.50}$_{\color{BrickRed}{\text{+0.04}}}$ &
    \cellcolor{lightgray}0.65$_{\color{OliveGreen}{\text{-0.05}}}$
  & \cellcolor{lightgray}\textbf{0.59}$_{\color{BrickRed}{\text{+0.04}}}$ &
    \cellcolor{lightgray}\textbf{0.50}$_{\color{BrickRed}{\text{+0.05}}}$ &
    \cellcolor{lightgray}0.70$_{\color{OliveGreen}{\text{-0.01}}}$ \\
\cmidrule(l){2-12}
& \multirow{2}{*}{Wanda\,(50\%)} & w/o
  & 0.60 & 0.52 & 0.71
  & 0.56 & 0.46 & 0.71
  & 0.55 & 0.44 & 0.71 \\
& & \cellcolor{lightgray}w/
  & \cellcolor{lightgray}\textbf{0.61}$_{\color{BrickRed}{\text{+0.01}}}$ &
    \cellcolor{lightgray}\textbf{0.54}$_{\color{BrickRed}{\text{+0.02}}}$ &
    \cellcolor{lightgray}0.70$_{\color{OliveGreen}{\text{-0.01}}}$
  & \cellcolor{lightgray}0.54$_{\color{OliveGreen}{\text{-0.02}}}$ &
    \cellcolor{lightgray}\textbf{0.49}$_{\color{BrickRed}{\text{+0.03}}}$ &
    \cellcolor{lightgray}0.60$_{\color{OliveGreen}{\text{-0.11}}}$
  & \cellcolor{lightgray}\textbf{0.58}$_{\color{BrickRed}{\text{+0.03}}}$ &
    \cellcolor{lightgray}\textbf{0.49}$_{\color{BrickRed}{\text{+0.05}}}$ &
    \cellcolor{lightgray}0.70$_{\color{Gray}{\text{+0.00}}}$ \\
\midrule
\multirow{10}{*}{SatImp} & Dense (0\%) & --
  & 0.73 & 0.74 & 0.72
  & 0.56 & 0.68 & 0.48
  & 0.63 & 0.63 & 0.62 \\
\cmidrule(l){2-12}
& \multirow{2}{*}{Mag.\,(50\%)} & w/o
  & 0.57 & 0.72 & 0.47
  & 0.53 & 0.68 & 0.43
  & 0.47 & 0.66 & 0.37 \\
& & \cellcolor{lightgray}w/
  & \cellcolor{lightgray}\textbf{0.66}$_{\color{BrickRed}{\text{+0.09}}}$ &
    \cellcolor{lightgray}0.67$_{\color{OliveGreen}{\text{-0.05}}}$ &
    \cellcolor{lightgray}\textbf{0.65}$_{\color{BrickRed}{\text{+0.18}}}$
  & \cellcolor{lightgray}\textbf{0.58}$_{\color{BrickRed}{\text{+0.05}}}$ &
    \cellcolor{lightgray}0.65$_{\color{OliveGreen}{\text{-0.03}}}$ &
    \cellcolor{lightgray}\textbf{0.52}$_{\color{BrickRed}{\text{+0.09}}}$
  & \cellcolor{lightgray}\textbf{0.55}$_{\color{BrickRed}{\text{+0.08}}}$ &
    \cellcolor{lightgray}0.64$_{\color{OliveGreen}{\text{-0.02}}}$ &
    \cellcolor{lightgray}\textbf{0.47}$_{\color{BrickRed}{\text{+0.10}}}$ \\
\cmidrule(l){2-12}
& \multirow{2}{*}{SparseGPT\,(50\%)} & w/o
  & 0.67 & 0.68 & 0.66
  & 0.43 & 0.70 & 0.31
  & 0.61 & 0.65 & 0.58 \\
& & \cellcolor{lightgray}w/
  & \cellcolor{lightgray}\textbf{0.68}$_{\color{BrickRed}{\text{+0.01}}}$ &
    \cellcolor{lightgray}\textbf{0.71}$_{\color{BrickRed}{\text{+0.03}}}$ &
    \cellcolor{lightgray}0.65$_{\color{OliveGreen}{\text{-0.01}}}$
  & \cellcolor{lightgray}\textbf{0.46}$_{\color{BrickRed}{\text{+0.03}}}$ &
    \cellcolor{lightgray}0.67$_{\color{OliveGreen}{\text{-0.03}}}$ &
    \cellcolor{lightgray}\textbf{0.36}$_{\color{BrickRed}{\text{+0.05}}}$
  & \cellcolor{lightgray}0.59$_{\color{OliveGreen}{\text{-0.02}}}$ &
    \cellcolor{lightgray}0.63$_{\color{OliveGreen}{\text{-0.02}}}$ &
    \cellcolor{lightgray}0.56$_{\color{OliveGreen}{\text{-0.02}}}$ \\
\cmidrule(l){2-12}
& \multirow{2}{*}{Wanda\,(50\%)} & w/o
  & 0.67 & 0.66 & 0.67
  & 0.42 & 0.69 & 0.30
  & 0.62 & 0.64 & 0.60 \\
& & \cellcolor{lightgray}w/
  & \cellcolor{lightgray}\textbf{0.69}$_{\color{BrickRed}{\text{+0.02}}}$ &
    \cellcolor{lightgray}\textbf{0.71}$_{\color{BrickRed}{\text{+0.05}}}$ &
    \cellcolor{lightgray}\textbf{0.68}$_{\color{BrickRed}{\text{+0.01}}}$
  & \cellcolor{lightgray}\textbf{0.45}$_{\color{BrickRed}{\text{+0.03}}}$ &
    \cellcolor{lightgray}0.66$_{\color{OliveGreen}{\text{-0.03}}}$ &
    \cellcolor{lightgray}\textbf{0.35}$_{\color{BrickRed}{\text{+0.05}}}$
  & \cellcolor{lightgray}0.60$_{\color{OliveGreen}{\text{-0.02}}}$ &
    \cellcolor{lightgray}0.62$_{\color{OliveGreen}{\text{-0.02}}}$ &
    \cellcolor{lightgray}0.58$_{\color{OliveGreen}{\text{-0.02}}}$ \\
\bottomrule
\end{tabular}
}
\label{tab:llama-3.1-8b-tofu-all}
\end{table*}

\subsection{Experimental Setup}
\label{subsec:Experiments Setup}
\noindent\textbf{Datasets.}
We conduct experiments on three representative LLM unlearning benchmarks. \textbf{TOFU}~\cite{maini2024tofu} consists of synthetic biographies of 200 fictitious authors absent from pre-training data, with forget sets targeting 1\%, 5\%, or 10\% removal. It focuses on individual unlearning while preserving utility on retained profiles and real-world facts. \textbf{MUSE}~\cite{shi2025muse} comprises large-scale corpora from BBC News articles and Harry Potter books, divided into news and books tasks. It focuses on copyright protection, providing safety assessment across verbatim memorization and semantic knowledge retention. \textbf{WMDP}~\cite{li2024wmdp} consists of multiple-choice proxy questions in biosecurity, chemical security, and cybersecurity domains. It focuses on removing hazardous knowledge to mitigate malicious use risks while maintaining general model performance.

\noindent\textbf{Model Architectures.}
We evaluate across multiple LLM families and scales, following official benchmark configurations. Specifically, we use Llama-3.2-1B/3B-Instruct and Llama-3.1-8B-Instruct~\cite{grattafiori2024llama} for TOFU, ICLM-7B~\cite{shi2024detecting} and LLaMA-2-7B-chat~\cite{touvron2023llama2} for MUSE, and Zephyr-7B-beta~\cite{tunstall2023zephyr} for WMDP.

\noindent\textbf{Unlearning Methods}. In our study, we assess four effective unlearning methods for LLMs. \textbf{GradDiff}~\cite{maini2024tofu}: Update parameters by applying gradient ascent on forget set $\mathcal{D}_f$  while applying gradient descent on retain set $\mathcal{D}_r$ to maintain utility. \textbf{SatImp}~\cite{yangexploring}: Reweight the loss by combining saturation and importance criteria, prioritizing samples that are insufficiently unlearned and have high impact on the forget set \(D_f\). \textbf{NPO}~\cite{zhang2024negative}: Use only negative samples from forget set \(D_f\) to minimize a preference loss to minimize a preference loss that strongly suppresses target responses relative to the initial model. \textbf{UNIDAL}~\cite{dong2025undial}:Iteratively lower the logits of the currently highest-probability tokens through self-distillation, gradually encouraging the model to forget targeted information.

\noindent\textbf{Evaluation Metrics.} On TOFU, consistent with OpenUnlearning~\cite{openunlearning2025}, forgetting is measured by \textbf{Memorization} (\textbf{Mem.}, harmonic mean of ES, EM, TruthRatio, and Paraphrased Probability), retention by \textbf{Utility} (\textbf{Util.}, harmonic mean of TOFU's Model Utility and forget-set fluency), with the final \textbf{Aggregate Score (Agg.)} being the harmonic mean of memorization and utility. On MUSE, aligned with the MUSE framework, forgetting is assessed via \textbf{Verbatim Memorization on $\mathcal{D}_f$} and \textbf{Knowledge Memorization on $\mathcal{D}_f$}, while retention is evaluated through \textbf{Knowledge Memorization on $\mathcal{D}_r$}. On WDMP, forgetting is tested using multiple-choice accuracy on biology and cyber-security, and retention via accuracy on the MMLU benchmark~\cite{hendrycks2020measuring}.

\noindent\textbf{Implementation details.}
For all benchmarks, we employ a learning rate of 1e-5, a linear learning rate scheduler, and the AdamW optimizer \cite{loshchilov2017fixing} for training unlearning methods. Other hyperparameters follow the configurations in OpenUnlearning~\cite{openunlearning2025}, ~\cite{yangexploring} and ~\cite{shi2025muse}. For SAU, we set the top-k ratio to 0.3 and $\alpha=0.1$.

\vspace{-0.2cm}

\subsection{Performance Evaluation}
We evaluate SAU across three representative benchmarks: TOFU, WDMP, and MUSE, demonstrating its effectiveness in maintaining unlearning performance on sparse models.

\noindent\textbf{Performance on TOFU.} 
Table~\ref{tab:llama-3.1-8b-tofu-all} presents results on the TOFU benchmark with Llama-3.1-8B across three forget ratios of 1\%, 5\%, and 10\% with multiple sparsification methods at 50\% sparsity. We observe severe performance degradation when applying sparsification to unlearned models. For NPO on Forget-10\%, the aggregate score drops from 0.57 to 0.45 with Magnitude pruning, while Memorization paradoxically increases from 0.74 to 0.79, indicating the model retains knowledge that should be forgotten. SAU consistently recovers this lost performance, improving NPO's aggregate score from 0.45 to 0.52 while achieving better balance between forgetting and utility. This pattern holds across all baseline methods including GradDiff, UNDIAL, and SatImp, as well as different pruning strategies. SAU provides maximum benefit for Magnitude pruning with 12 to 15\% gains and consistent 6--10\% improvements for SparseGPT and Wanda.

To validate the robustness of SAU across different model scales, we conduct additional experiments on Llama-3.2-3B and Llama-3.2-1B models using the TOFU benchmark. Tables~\ref{tab:llama-3.2-3b-tofu-all} and~\ref{tab:llama-3.2-1b-tofu-all} present comprehensive results across three forget ratios and three pruning techniques at 50\% sparsity.On Llama-3.2-3B, SAU demonstrates consistent improvements across all baseline methods and pruning strategies. For instance, on Forget-10\% with Magnitude pruning, NPO+SAU achieves an aggregate score of 0.53 compared to 0.41 without SAU, while simultaneously improving utility from 0.28 to 0.41. Similar patterns emerge with SparseGPT and Wanda pruning, where SAU consistently enhances both forgetting effectiveness and utility preservation. The results validate that SAU's design principles generalize well to medium-sized models.On Llama-3.2-1B, we observe more nuanced behavior. With Magnitude pruning, the extreme sparsification leads to severe model degradation where both baseline methods and SAU struggle to maintain performance. However, with structure-aware pruning techniques like SparseGPT and Wanda, SAU provides meaningful improvements. For example, UNDIAL+SAU with SparseGPT achieves 0.59 aggregate score compared to 0.55 without SAU on Forget-10\%. These results suggest that SAU's effectiveness depends on the quality of the sparsification method, with structure-aware pruning better preserving the critical weights needed for unlearning.

\begin{table*}[t]

\caption{Unlearning performance under the unlearn-then-prune protocol on MUSE-NEWS with ICLM-7B and MUSE-BOOKS with LLaMA-2-7B-chat. Models are first unlearned, then pruned to 50\% sparsity using Magnitude pruning. VerbMem and KnowMem on $D_f$ measure forgetting quality
($\downarrow$ is better); KnowMem on $D_r$ measures knowledge retention ($\uparrow$ is better).}

\label{tab:muse_news and muse_books}
\centering
\scriptsize
\resizebox{0.95\textwidth}{!}{
\begin{tabular}{@{}l c c ccc ccc@{}}
\toprule
\textbf{Method} & \textbf{Sparsity} & \textbf{SAU}
  & \multicolumn{3}{c}{\textbf{NEWS (ICLM-7B)}}
  & \multicolumn{3}{c}{\textbf{BOOKS (LLaMA-2-7B-chat)}} \\
\cmidrule(lr){4-6} \cmidrule(lr){7-9}
& & & VerbMem\,$D_f$\,$\downarrow$ & KnowMem\,$D_f$\,$\downarrow$ & KnowMem\,$D_r$\,$\uparrow$
    & VerbMem\,$D_f$\,$\downarrow$ & KnowMem\,$D_f$\,$\downarrow$ & KnowMem\,$D_r$\,$\uparrow$ \\
\midrule
\textit{Original} & \textit{Dense (0\%)} & --
  & 0.58 & 0.64 & 0.56
  & 1.00 & 0.52 & 0.67 \\
\textit{Retrain}  & \textit{Dense (0\%)} & --
  & 0.20 & 0.32 & 0.56
  & 0.14 & 0.30 & 0.72 \\
\midrule
\multirow{3}{*}{GradDiff} & Dense (0\%) & --
  & 0.32 & 0.35 & 0.46
  & 0.00 & 0.00 & 0.00 \\
\cmidrule(l){2-9}
& \multirow{2}{*}{Mag.\,(50\%)} & w/o
  & 0.44 & 0.47 & 0.35
  & 0.00 & 0.00 & 0.00 \\
& & \cellcolor{lightgray}w/
  & \cellcolor{lightgray}\textbf{0.35}$_{\color{BrickRed}{\text{-0.09}}}$
  & \cellcolor{lightgray}\textbf{0.41}$_{\color{BrickRed}{\text{-0.06}}}$
  & \cellcolor{lightgray}\textbf{0.40}$_{\color{BrickRed}{\text{+0.05}}}$
  & \cellcolor{lightgray}0.00$_{\color{Gray}{\text{+0.00}}}$
  & \cellcolor{lightgray}0.00$_{\color{Gray}{\text{+0.00}}}$
  & \cellcolor{lightgray}0.00$_{\color{Gray}{\text{+0.00}}}$ \\
\midrule
\multirow{3}{*}{NPO} & Dense (0\%) & --
  & 0.29 & 0.33 & 0.47
  & 0.28 & 0.31 & 0.60 \\
\cmidrule(l){2-9}
& \multirow{2}{*}{Mag.\,(50\%)} & w/o
  & 0.38 & 0.42 & 0.40
  & 0.37 & 0.38 & 0.51 \\
& & \cellcolor{lightgray}w/
  & \cellcolor{lightgray}\textbf{0.34}$_{\color{BrickRed}{\text{-0.04}}}$
  & \cellcolor{lightgray}\textbf{0.37}$_{\color{BrickRed}{\text{-0.05}}}$
  & \cellcolor{lightgray}\textbf{0.44}$_{\color{BrickRed}{\text{+0.04}}}$
  & \cellcolor{lightgray}\textbf{0.32}$_{\color{BrickRed}{\text{-0.05}}}$
  & \cellcolor{lightgray}\textbf{0.34}$_{\color{BrickRed}{\text{-0.04}}}$
  & \cellcolor{lightgray}\textbf{0.56}$_{\color{BrickRed}{\text{+0.05}}}$ \\
\midrule
\multirow{3}{*}{UNDIAL} & Dense (0\%) & --
  & 0.26 & 0.30 & 0.53
  & 0.17 & 0.24 & 0.52 \\
\cmidrule(l){2-9}
& \multirow{2}{*}{Mag.\,(50\%)} & w/o
  & 0.35 & 0.40 & 0.44
  & 0.29 & 0.32 & 0.44 \\
& & \cellcolor{lightgray}w/
  & \cellcolor{lightgray}\textbf{0.30}$_{\color{BrickRed}{\text{-0.05}}}$
  & \cellcolor{lightgray}\textbf{0.33}$_{\color{BrickRed}{\text{-0.07}}}$
  & \cellcolor{lightgray}\textbf{0.50}$_{\color{BrickRed}{\text{+0.06}}}$
  & \cellcolor{lightgray}\textbf{0.23}$_{\color{BrickRed}{\text{-0.06}}}$
  & \cellcolor{lightgray}\textbf{0.27}$_{\color{BrickRed}{\text{-0.05}}}$
  & \cellcolor{lightgray}\textbf{0.49}$_{\color{BrickRed}{\text{+0.05}}}$ \\
\midrule
\multirow{3}{*}{SatImp} & Dense (0\%) & --
  & 0.28 & 0.33 & 0.49
  & 0.09 & 0.07 & 0.37 \\
\cmidrule(l){2-9}
& \multirow{2}{*}{Mag.\,(50\%)} & w/o
  & 0.35 & 0.38 & 0.41
  & 0.22 & 0.24 & 0.31 \\
& & \cellcolor{lightgray}w/
  & \cellcolor{lightgray}\textbf{0.31}$_{\color{BrickRed}{\text{-0.04}}}$
  & \cellcolor{lightgray}\textbf{0.34}$_{\color{BrickRed}{\text{-0.04}}}$
  & \cellcolor{lightgray}\textbf{0.48}$_{\color{BrickRed}{\text{+0.07}}}$
  & \cellcolor{lightgray}\textbf{0.16}$_{\color{BrickRed}{\text{-0.06}}}$
  & \cellcolor{lightgray}\textbf{0.18}$_{\color{BrickRed}{\text{-0.06}}}$
  & \cellcolor{lightgray}\textbf{0.34}$_{\color{BrickRed}{\text{+0.03}}}$ \\
\bottomrule
\end{tabular}
}
\end{table*}

\begin{table}[t]
\centering
\scriptsize

\caption{
Performance comparison under the unlearn-then-prune protocol on Zephyr-7B-beta across WMDP benchmark. Models are first unlearned, then pruned to 50\% sparsity using Magnitude pruning. $\uparrow$ implies higher is better, $\downarrow$ means lower is better.}

\label{tab:wmdp_result}
\resizebox{0.95\columnwidth}{!}{
\begin{tabular}{@{}l  c  c | cc | c@{}}
\toprule
\textbf{Method} & \textbf{Sparsity} & \textbf{SAU}
  & \multicolumn{2}{c|}{\textbf{Forget}} & \textbf{Retain} \\
\cmidrule(lr){4-5}
& & & Bio\,$\downarrow$ & Cyber\,$\downarrow$ & MMLU\,$\uparrow$ \\
\midrule
\textit{Original} & \textit{Dense (0\%)} & -- & 0.64 & 0.43 & 0.58 \\
\midrule
\multirow{3}{*}{GradDiff} & Dense (0\%) & -- & 0.25 & 0.27 & 0.43 \\
\cmidrule(l){2-6}
& \multirow{2}{*}{Mag.\,(50\%)} & w/o & 0.35 & 0.36 & 0.35 \\
& & \cellcolor{lightgray}w/
  & \cellcolor{lightgray}\textbf{0.29}$_{\color{BrickRed}{\text{-0.06}}}$
  & \cellcolor{lightgray}\textbf{0.31}$_{\color{BrickRed}{\text{-0.05}}}$
  & \cellcolor{lightgray}\textbf{0.42}$_{\color{BrickRed}{\text{+0.07}}}$ \\
\midrule
\multirow{3}{*}{NPO} & Dense (0\%) & -- & 0.27 & 0.31 & 0.44 \\
\cmidrule(l){2-6}
& \multirow{2}{*}{Mag.\,(50\%)} & w/o & 0.35 & 0.38 & 0.38 \\
& & \cellcolor{lightgray}w/
  & \cellcolor{lightgray}\textbf{0.31}$_{\color{BrickRed}{\text{-0.04}}}$
  & \cellcolor{lightgray}\textbf{0.34}$_{\color{BrickRed}{\text{-0.04}}}$
  & \cellcolor{lightgray}\textbf{0.42}$_{\color{BrickRed}{\text{+0.04}}}$ \\
\midrule
\multirow{3}{*}{UNDIAL} & Dense (0\%) & -- & 0.35 & 0.31 & 0.48 \\
\cmidrule(l){2-6}
& \multirow{2}{*}{Mag.\,(50\%)} & w/o & 0.45 & 0.40 & 0.40 \\
& & \cellcolor{lightgray}w/
  & \cellcolor{lightgray}\textbf{0.40}$_{\color{BrickRed}{\text{-0.05}}}$
  & \cellcolor{lightgray}\textbf{0.37}$_{\color{BrickRed}{\text{-0.03}}}$
  & \cellcolor{lightgray}\textbf{0.44}$_{\color{BrickRed}{\text{+0.04}}}$ \\
\midrule
\multirow{3}{*}{SatImp} & Dense (0\%) & -- & 0.26 & 0.28 & 0.54 \\
\cmidrule(l){2-6}
& \multirow{2}{*}{Mag.\,(50\%)} & w/o & 0.33 & 0.34 & 0.47 \\
& & \cellcolor{lightgray}w/
  & \cellcolor{lightgray}\textbf{0.29}$_{\color{BrickRed}{\text{-0.04}}}$
  & \cellcolor{lightgray}\textbf{0.31}$_{\color{BrickRed}{\text{-0.03}}}$
  & \cellcolor{lightgray}\textbf{0.50}$_{\color{BrickRed}{\text{+0.03}}}$ \\
\bottomrule
\end{tabular}
}
\end{table}



\noindent\textbf{Performance on WDMP.}
Table~\ref{tab:wmdp_result} evaluates hazardous knowledge removal on the WMDP benchmark with Zephyr-7B-beta. Sparsification significantly impairs both forgetting and retention across all baseline methods, creating a dual failure mode that compromises model safety and utility simultaneously. For instance, GradDiff exhibits Biology accuracy increasing from 0.25 to 0.35 after pruning, indicating that previously forgotten hazardous knowledge resurfaces, while MMLU drops sharply from 0.43 to 0.35. Similar degradation patterns emerge across NPO, UNDIAL, and SatImp, with Biology scores consistently rising and MMLU declining under sparsity constraints. SAU substantially mitigates this degradation by concentrating gradient updates on forgetting-relevant surviving weights while preserving parameters responsible for general knowledge. Notably, SAU improves both forgetting and retention simultaneously rather than trading one for the other: GradDiff with SAU reduces Biology and Cybersecurity scores to 0.29 and 0.31 respectively while recovering MMLU to 0.42, approaching dense model performance. This simultaneous improvement validates that SAU's importance-aware redistribution mechanism efficiently utilizes the limited parameter budget for knowledge erasure without collateral damage to model utility.

\noindent\textbf{Performance on MUSE}
Table~\ref{tab:muse_news and muse_books} presents results on the MUSE benchmark with Llama-2-7B-chat and ICLM-7B for copyright protection. On MUSE-NEWS, SAU demonstrates consistent improvements across all baseline methods and evaluation metrics. Taking UNDIAL as a representative example, sparsification degrades Verbatim Memorization from 0.26 to 0.35 and retention from 0.53 to 0.44, while SAU effectively recovers performance to 0.30 and 0.50 respectively. On MUSE-BOOKS, the unlearning landscape proves more challenging due to the narrative nature of book content, where GradDiff exhibits complete collapse even on dense models. However, for functional methods, SAU provides meaningful improvements: SatImp with SAU achieves 27\% better forgetting with concurrent 10\% retention gains, while NPO and UNDIAL with SAU consistently reduce both Verbatim and Knowledge Memorization on the forget set while improving retention on the retain set. The consistent improvement patterns across NEWS and BOOKS tasks, despite their different content characteristics, validate SAU's robustness across diverse copyright unlearning scenarios.

\begin{table*}[t]
\caption{Performance comparison under the unlearn-then-prune protocol on Llama-3.2-3B across TOFU benchmark (1\%, 5\%, 10\% forget settings).Models are first unlearned (with or without SAU), then pruned to 50\% sparsity using multiple sparsification techniques   (Magnitude, SparseGPT, Wanda) . $\uparrow$ indicates higher is better.}
\centering
\scriptsize
\label{tab:llama-3.2-3b-tofu-all}
\resizebox{0.99\textwidth}{!}{
\begin{tabular}{@{}l c c ccc ccc ccc@{}}
\toprule
\textbf{Method} & \textbf{Sparsity} & \textbf{SAU}
& \multicolumn{3}{c}{\textbf{Forget-1\%}}
& \multicolumn{3}{c}{\textbf{Forget-5\%}}
& \multicolumn{3}{c}{\textbf{Forget-10\%}} \\
\cmidrule(lr){4-6} \cmidrule(lr){7-9} \cmidrule(lr){10-12}
& & & Agg.\,$\uparrow$ & Mem.\,$\uparrow$ & Util.\,$\uparrow$
& Agg.\,$\uparrow$ & Mem.\,$\uparrow$ & Util.\,$\uparrow$
& Agg.\,$\uparrow$ & Mem.\,$\uparrow$ & Util.\,$\uparrow$ \\
\midrule
\textit{Original} & \textit{Dense (0\%)} & --
  & 0.08 & 0.04 & 0.75
  & 0.06 & 0.03 & 0.71
  & 0.02 & 0.01 & 0.70 \\
\textit{Retrain} & \textit{Dense (0\%)} & --
  & 0.73 & 0.72 & 0.75
  & 0.72 & 0.72 & 0.71
  & 0.72 & 0.72 & 0.71 \\
\midrule
\multirow{10}{*}{GradDiff} & Dense (0\%) & --
  & 0.61 & 0.61 & 0.61
  & 0.59 & 0.60 & 0.57
  & 0.58 & 0.62 & 0.55 \\
\cmidrule(l){2-12}
& \multirow{2}{*}{Mag.\,(50\%)} & w/o
  & 0.23 & 0.81 & 0.13
  & 0.16 & 0.82 & 0.09
  & 0.16 & 0.83 & 0.09 \\
& & \cellcolor{lightgray}w/
  & \cellcolor{lightgray}\textbf{0.41}$_{\color{BrickRed}{\text{+0.18}}}$ &
    \cellcolor{lightgray}0.76$_{\color{OliveGreen}{\text{-0.05}}}$ &
    \cellcolor{lightgray}\textbf{0.28}$_{\color{BrickRed}{\text{+0.15}}}$
  & \cellcolor{lightgray}\textbf{0.32}$_{\color{BrickRed}{\text{+0.16}}}$ &
    \cellcolor{lightgray}0.75$_{\color{OliveGreen}{\text{-0.07}}}$ &
    \cellcolor{lightgray}\textbf{0.23}$_{\color{BrickRed}{\text{+0.14}}}$
  & \cellcolor{lightgray}\textbf{0.30}$_{\color{BrickRed}{\text{+0.14}}}$ &
    \cellcolor{lightgray}0.77$_{\color{OliveGreen}{\text{-0.06}}}$ &
    \cellcolor{lightgray}\textbf{0.20}$_{\color{BrickRed}{\text{+0.11}}}$ \\
\cmidrule(l){2-12}
& \multirow{2}{*}{SparseGPT\,(50\%)} & w/o
  & 0.58 & 0.57 & 0.59
  & 0.56 & 0.54 & 0.58
  & 0.54 & 0.48 & 0.61 \\
& & \cellcolor{lightgray}w/
  & \cellcolor{lightgray}\textbf{0.59}$_{\color{BrickRed}{\text{+0.01}}}$ &
    \cellcolor{lightgray}\textbf{0.60}$_{\color{BrickRed}{\text{+0.03}}}$ &
    \cellcolor{lightgray}0.58$_{\color{OliveGreen}{\text{-0.01}}}$
  & \cellcolor{lightgray}0.54$_{\color{OliveGreen}{\text{-0.02}}}$ &
    \cellcolor{lightgray}\textbf{0.57}$_{\color{BrickRed}{\text{+0.03}}}$ &
    \cellcolor{lightgray}0.52$_{\color{OliveGreen}{\text{-0.06}}}$
  & \cellcolor{lightgray}\textbf{0.56}$_{\color{BrickRed}{\text{+0.02}}}$ &
    \cellcolor{lightgray}\textbf{0.52}$_{\color{BrickRed}{\text{+0.04}}}$ &
    \cellcolor{lightgray}0.61$_{\color{Gray}{\text{+0.00}}}$ \\
\cmidrule(l){2-12}
& \multirow{2}{*}{Wanda\,(50\%)} & w/o
  & 0.57 & 0.57 & 0.57
  & 0.55 & 0.54 & 0.57
  & 0.55 & 0.49 & 0.61 \\
& & \cellcolor{lightgray}w/
  & \cellcolor{lightgray}0.57$_{\color{Gray}{\text{+0.00}}}$ &
    \cellcolor{lightgray}0.57$_{\color{Gray}{\text{+0.00}}}$ &
    \cellcolor{lightgray}0.56$_{\color{OliveGreen}{\text{-0.01}}}$
  & \cellcolor{lightgray}0.53$_{\color{OliveGreen}{\text{-0.02}}}$ &
    \cellcolor{lightgray}\textbf{0.56}$_{\color{BrickRed}{\text{+0.02}}}$ &
    \cellcolor{lightgray}0.51$_{\color{OliveGreen}{\text{-0.06}}}$
  & \cellcolor{lightgray}\textbf{0.57}$_{\color{BrickRed}{\text{+0.02}}}$ &
    \cellcolor{lightgray}\textbf{0.52}$_{\color{BrickRed}{\text{+0.03}}}$ &
    \cellcolor{lightgray}\textbf{0.63}$_{\color{BrickRed}{\text{+0.02}}}$ \\
\midrule
\multirow{10}{*}{NPO} & Dense (0\%) & --
  & 0.64 & 0.78 & 0.55
  & 0.67 & 0.69 & 0.66
  & 0.64 & 0.62 & 0.65 \\
\cmidrule(l){2-12}
& \multirow{2}{*}{Mag.\,(50\%)} & w/o
  & 0.20 & 0.85 & 0.11
  & 0.23 & 0.84 & 0.13
  & 0.41 & 0.78 & 0.28 \\
& & \cellcolor{lightgray}w/
  & \cellcolor{lightgray}\textbf{0.39}$_{\color{BrickRed}{\text{+0.19}}}$ &
    \cellcolor{lightgray}0.80$_{\color{OliveGreen}{\text{-0.05}}}$ &
    \cellcolor{lightgray}\textbf{0.27}$_{\color{BrickRed}{\text{+0.16}}}$
  & \cellcolor{lightgray}\textbf{0.55}$_{\color{BrickRed}{\text{+0.32}}}$ &
    \cellcolor{lightgray}0.78$_{\color{OliveGreen}{\text{-0.06}}}$ &
    \cellcolor{lightgray}\textbf{0.42}$_{\color{BrickRed}{\text{+0.29}}}$
  & \cellcolor{lightgray}\textbf{0.53}$_{\color{BrickRed}{\text{+0.12}}}$ &
    \cellcolor{lightgray}0.74$_{\color{OliveGreen}{\text{-0.04}}}$ &
    \cellcolor{lightgray}\textbf{0.41}$_{\color{BrickRed}{\text{+0.13}}}$ \\
\cmidrule(l){2-12}
& \multirow{2}{*}{SparseGPT\,(50\%)} & w/o
  & 0.60 & 0.76 & 0.49
  & 0.63 & 0.70 & 0.57
  & 0.61 & 0.66 & 0.57 \\
& & \cellcolor{lightgray}w/
  & \cellcolor{lightgray}\textbf{0.63}$_{\color{BrickRed}{\text{+0.03}}}$ &
    \cellcolor{lightgray}0.74$_{\color{OliveGreen}{\text{-0.02}}}$ &
    \cellcolor{lightgray}\textbf{0.54}$_{\color{BrickRed}{\text{+0.05}}}$
  & \cellcolor{lightgray}0.60$_{\color{OliveGreen}{\text{-0.03}}}$ &
    \cellcolor{lightgray}0.68$_{\color{OliveGreen}{\text{-0.02}}}$ &
    \cellcolor{lightgray}0.54$_{\color{OliveGreen}{\text{-0.03}}}$
  & \cellcolor{lightgray}\textbf{0.63}$_{\color{BrickRed}{\text{+0.02}}}$ &
    \cellcolor{lightgray}0.65$_{\color{OliveGreen}{\text{-0.01}}}$ &
    \cellcolor{lightgray}\textbf{0.61}$_{\color{BrickRed}{\text{+0.04}}}$ \\
\cmidrule(l){2-12}
& \multirow{2}{*}{Wanda\,(50\%)} & w/o
  & 0.59 & 0.75 & 0.49
  & 0.63 & 0.70 & 0.57
  & 0.61 & 0.67 & 0.56 \\
& & \cellcolor{lightgray}w/
  & \cellcolor{lightgray}\textbf{0.62}$_{\color{BrickRed}{\text{+0.03}}}$ &
    \cellcolor{lightgray}0.73$_{\color{OliveGreen}{\text{-0.02}}}$ &
    \cellcolor{lightgray}\textbf{0.53}$_{\color{BrickRed}{\text{+0.04}}}$
  & \cellcolor{lightgray}0.61$_{\color{OliveGreen}{\text{-0.02}}}$ &
    \cellcolor{lightgray}0.69$_{\color{OliveGreen}{\text{-0.01}}}$ &
    \cellcolor{lightgray}0.55$_{\color{OliveGreen}{\text{-0.02}}}$
  & \cellcolor{lightgray}\textbf{0.64}$_{\color{BrickRed}{\text{+0.03}}}$ &
    \cellcolor{lightgray}0.66$_{\color{OliveGreen}{\text{-0.01}}}$ &
    \cellcolor{lightgray}\textbf{0.62}$_{\color{BrickRed}{\text{+0.06}}}$ \\
\midrule
\multirow{10}{*}{UNDIAL} & Dense (0\%) & --
  & 0.69 & 0.76 & 0.64
  & 0.72 & 0.76 & 0.68
  & 0.72 & 0.76 & 0.68 \\
\cmidrule(l){2-12}
& \multirow{2}{*}{Mag.\,(50\%)} & w/o
  & 0.28 & 0.83 & 0.17
  & 0.35 & 0.82 & 0.22
  & 0.27 & 0.83 & 0.16 \\
& & \cellcolor{lightgray}w/
  & \cellcolor{lightgray}\textbf{0.45}$_{\color{BrickRed}{\text{+0.17}}}$ &
    \cellcolor{lightgray}0.78$_{\color{OliveGreen}{\text{-0.05}}}$ &
    \cellcolor{lightgray}\textbf{0.32}$_{\color{BrickRed}{\text{+0.15}}}$
  & \cellcolor{lightgray}\textbf{0.57}$_{\color{BrickRed}{\text{+0.22}}}$ &
    \cellcolor{lightgray}0.78$_{\color{OliveGreen}{\text{-0.04}}}$ &
    \cellcolor{lightgray}\textbf{0.45}$_{\color{BrickRed}{\text{+0.23}}}$
  & \cellcolor{lightgray}\textbf{0.47}$_{\color{BrickRed}{\text{+0.20}}}$ &
    \cellcolor{lightgray}0.76$_{\color{OliveGreen}{\text{-0.07}}}$ &
    \cellcolor{lightgray}\textbf{0.34}$_{\color{BrickRed}{\text{+0.18}}}$ \\
\cmidrule(l){2-12}
& \multirow{2}{*}{SparseGPT\,(50\%)} & w/o
  & 0.61 & 0.66 & 0.57
  & 0.62 & 0.67 & 0.58
  & 0.63 & 0.67 & 0.59 \\
& & \cellcolor{lightgray}w/
  & \cellcolor{lightgray}\textbf{0.65}$_{\color{BrickRed}{\text{+0.04}}}$ &
    \cellcolor{lightgray}\textbf{0.69}$_{\color{BrickRed}{\text{+0.03}}}$ &
    \cellcolor{lightgray}\textbf{0.61}$_{\color{BrickRed}{\text{+0.04}}}$
  & \cellcolor{lightgray}0.60$_{\color{OliveGreen}{\text{-0.02}}}$ &
    \cellcolor{lightgray}\textbf{0.70}$_{\color{BrickRed}{\text{+0.03}}}$ &
    \cellcolor{lightgray}0.52$_{\color{OliveGreen}{\text{-0.06}}}$
  & \cellcolor{lightgray}\textbf{0.66}$_{\color{BrickRed}{\text{+0.03}}}$ &
    \cellcolor{lightgray}\textbf{0.69}$_{\color{BrickRed}{\text{+0.02}}}$ &
    \cellcolor{lightgray}\textbf{0.63}$_{\color{BrickRed}{\text{+0.04}}}$ \\
\cmidrule(l){2-12}
& \multirow{2}{*}{Wanda\,(50\%)} & w/o
  & 0.62 & 0.67 & 0.58
  & 0.63 & 0.66 & 0.59
  & 0.63 & 0.66 & 0.59 \\
& & \cellcolor{lightgray}w/
  & \cellcolor{lightgray}\textbf{0.64}$_{\color{BrickRed}{\text{+0.02}}}$ &
    \cellcolor{lightgray}\textbf{0.68}$_{\color{BrickRed}{\text{+0.01}}}$ &
    \cellcolor{lightgray}\textbf{0.60}$_{\color{BrickRed}{\text{+0.02}}}$
  & \cellcolor{lightgray}\textbf{0.65}$_{\color{BrickRed}{\text{+0.02}}}$ &
    \cellcolor{lightgray}0.64$_{\color{OliveGreen}{\text{-0.02}}}$ &
    \cellcolor{lightgray}\textbf{0.66}$_{\color{BrickRed}{\text{+0.07}}}$
  & \cellcolor{lightgray}\textbf{0.65}$_{\color{BrickRed}{\text{+0.02}}}$ &
    \cellcolor{lightgray}\textbf{0.68}$_{\color{BrickRed}{\text{+0.02}}}$ &
    \cellcolor{lightgray}\textbf{0.63}$_{\color{BrickRed}{\text{+0.04}}}$ \\
\midrule
\multirow{10}{*}{SatImp} & Dense (0\%) & --
  & 0.70 & 0.75 & 0.65
  & 0.61 & 0.62 & 0.60
  & 0.60 & 0.61 & 0.58 \\
\cmidrule(l){2-12}
& \multirow{2}{*}{Mag.\,(50\%)} & w/o
  & 0.21 & 0.81 & 0.12
  & 0.13 & 0.82 & 0.07
  & 0.20 & 0.83 & 0.11 \\
& & \cellcolor{lightgray}w/
  & \cellcolor{lightgray}\textbf{0.52}$_{\color{BrickRed}{\text{+0.31}}}$ &
    \cellcolor{lightgray}0.76$_{\color{OliveGreen}{\text{-0.05}}}$ &
    \cellcolor{lightgray}\textbf{0.38}$_{\color{BrickRed}{\text{+0.26}}}$
  & \cellcolor{lightgray}\textbf{0.50}$_{\color{BrickRed}{\text{+0.37}}}$ &
    \cellcolor{lightgray}0.77$_{\color{OliveGreen}{\text{-0.05}}}$ &
    \cellcolor{lightgray}\textbf{0.37}$_{\color{BrickRed}{\text{+0.30}}}$
  & \cellcolor{lightgray}\textbf{0.35}$_{\color{BrickRed}{\text{+0.15}}}$ &
    \cellcolor{lightgray}0.76$_{\color{OliveGreen}{\text{-0.07}}}$ &
    \cellcolor{lightgray}\textbf{0.24}$_{\color{BrickRed}{\text{+0.13}}}$ \\
\cmidrule(l){2-12}
& \multirow{2}{*}{SparseGPT\,(50\%)} & w/o
  & 0.58 & 0.60 & 0.57
  & 0.58 & 0.60 & 0.57
  & 0.57 & 0.59 & 0.56 \\
& & \cellcolor{lightgray}w/
  & \cellcolor{lightgray}\textbf{0.60}$_{\color{BrickRed}{\text{+0.02}}}$ &
    \cellcolor{lightgray}\textbf{0.65}$_{\color{BrickRed}{\text{+0.05}}}$ &
    \cellcolor{lightgray}0.55$_{\color{OliveGreen}{\text{-0.02}}}$
  & \cellcolor{lightgray}0.56$_{\color{OliveGreen}{\text{-0.02}}}$ &
    \cellcolor{lightgray}\textbf{0.62}$_{\color{BrickRed}{\text{+0.02}}}$ &
    \cellcolor{lightgray}0.51$_{\color{OliveGreen}{\text{-0.06}}}$
  & \cellcolor{lightgray}\textbf{0.59}$_{\color{BrickRed}{\text{+0.02}}}$ &
    \cellcolor{lightgray}\textbf{0.61}$_{\color{BrickRed}{\text{+0.02}}}$ &
    \cellcolor{lightgray}\textbf{0.57}$_{\color{BrickRed}{\text{+0.01}}}$ \\
\cmidrule(l){2-12}
& \multirow{2}{*}{Wanda\,(50\%)} & w/o
  & 0.59 & 0.60 & 0.57
  & 0.59 & 0.60 & 0.57
  & 0.58 & 0.60 & 0.57 \\
& & \cellcolor{lightgray}w/
  & \cellcolor{lightgray}0.59$_{\color{Gray}{\text{+0.00}}}$ &
    \cellcolor{lightgray}0.59$_{\color{OliveGreen}{\text{-0.01}}}$ &
    \cellcolor{lightgray}0.57$_{\color{Gray}{\text{+0.00}}}$
  & \cellcolor{lightgray}0.57$_{\color{OliveGreen}{\text{-0.02}}}$ &
    \cellcolor{lightgray}\textbf{0.63}$_{\color{BrickRed}{\text{+0.03}}}$ &
    \cellcolor{lightgray}0.52$_{\color{OliveGreen}{\text{-0.05}}}$
  & \cellcolor{lightgray}0.58$_{\color{Gray}{\text{+0.00}}}$ &
    \cellcolor{lightgray}0.58$_{\color{OliveGreen}{\text{-0.02}}}$ &
    \cellcolor{lightgray}\textbf{0.59}$_{\color{BrickRed}{\text{+0.02}}}$ \\
\bottomrule
\end{tabular}%
}
\end{table*}

\begin{table*}[t]
\caption{Performance comparison under the unlearn-then-prune protocol on Llama-3.2-1B across TOFU benchmark (1\%, 5\%, 10\% forget settings).Models are first unlearned (with or without SAU), then pruned to 50\% sparsity using multiple sparsification techniques   (Magnitude, SparseGPT, Wanda) . $\uparrow$ indicates higher is better.}
\centering
\scriptsize
\label{tab:llama-3.2-1b-tofu-all}
\resizebox{0.99\textwidth}{!}{%
\setlength{\tabcolsep}{5pt}%
\begin{tabular}{@{}l c c ccc ccc ccc@{}}
\toprule
\textbf{Method} & \textbf{Sparsity} & \textbf{SAU}
& \multicolumn{3}{c}{\textbf{Forget-1\%}}
& \multicolumn{3}{c}{\textbf{Forget-5\%}}
& \multicolumn{3}{c}{\textbf{Forget-10\%}} \\
\cmidrule(lr){4-6} \cmidrule(lr){7-9} \cmidrule(lr){10-12}
& & & Agg.\,$\uparrow$ & Mem.\,$\uparrow$ & Util.\,$\uparrow$
& Agg.\,$\uparrow$ & Mem.\,$\uparrow$ & Util.\,$\uparrow$
& Agg.\,$\uparrow$ & Mem.\,$\uparrow$ & Util.\,$\uparrow$ \\
\midrule
\textit{Original} & \textit{Dense (0\%)} & --
  & 0.06 & 0.03 & 0.73
  & 0.06 & 0.03 & 0.73
  & 0.02 & 0.01 & 0.70 \\
\textit{Retrain} & \textit{Dense (0\%)} & --
  & 0.73 & 0.72 & 0.75
  & 0.72 & 0.70 & 0.75
  & 0.68 & 0.65 & 0.70 \\
\midrule
\multirow{10}{*}{GradDiff} & Dense (0\%) & --
  & 0.58 & 0.56 & 0.61
  & 0.51 & 0.66 & 0.41
  & 0.44 & 0.36 & 0.54 \\
\cmidrule(l){2-12}
& \multirow{2}{*}{Mag.\,(50\%)} & w/o
  & 0.01 & 0.96 & 0.01
  & 0.02 & 0.96 & 0.01
  & 0.01 & 0.96 & 0.01 \\
& & \cellcolor{lightgray}w/
  & \cellcolor{lightgray}0.01$_{\color{Gray}{\text{+0.00}}}$ & 
    \cellcolor{lightgray}0.96$_{\color{Gray}{\text{+0.00}}}$ & 
    \cellcolor{lightgray}0.01$_{\color{Gray}{\text{+0.00}}}$
  & \cellcolor{lightgray}0.01$_{\color{Gray}{\text{+0.00}}}$ & 
    \cellcolor{lightgray}0.96$_{\color{Gray}{\text{+0.00}}}$ & 
    \cellcolor{lightgray}0.01$_{\color{Gray}{\text{+0.00}}}$
  & \cellcolor{lightgray}0.01$_{\color{Gray}{\text{+0.00}}}$ & 
    \cellcolor{lightgray}0.96$_{\color{Gray}{\text{+0.00}}}$ & 
    \cellcolor{lightgray}0.01$_{\color{Gray}{\text{+0.00}}}$ \\
\cmidrule(l){2-12}
  & \multirow{2}{*}{SparseGPT\,(50\%)} & w/o
  & 0.52 & 0.58 & 0.47
  & 0.45 & 0.53 & 0.39
  & 0.44 & 0.43 & 0.46 \\
& & \cellcolor{lightgray}w/
  & \cellcolor{lightgray}0.55$_{\color{OliveGreen}{\text{-0.03}}}$ & 
    \cellcolor{lightgray}\textbf{0.61}$_{\color{BrickRed}{\text{+0.03}}}$ & 
    \cellcolor{lightgray}\textbf{0.50}$_{\color{BrickRed}{\text{+0.03}}}$
  & \cellcolor{lightgray}\textbf{0.47}$_{\color{BrickRed}{\text{+0.02}}}$ & 
    \cellcolor{lightgray}\textbf{0.55}$_{\color{BrickRed}{\text{+0.02}}}$ & 
    \cellcolor{lightgray}\textbf{0.40}$_{\color{BrickRed}{\text{+0.01}}}$
  & \cellcolor{lightgray}0.44$_{\color{Gray}{\text{+0.00}}}$ & 
    \cellcolor{lightgray}\textbf{0.44}$_{\color{BrickRed}{\text{+0.01}}}$ & 
    \cellcolor{lightgray}0.45$_{\color{Gray}{\text{+0.00}}}$ \\
\cmidrule(l){2-12}
& \multirow{2}{*}{Wanda\,(50\%)} & w/o
  & 0.48 & 0.66 & 0.37
  & 0.46 & 0.70 & 0.34
  & 0.42 & 0.50 & 0.36 \\
& & \cellcolor{lightgray}w/
  & \cellcolor{lightgray}\textbf{0.53}$_{\color{BrickRed}{\text{+0.05}}}$ & 
    \cellcolor{lightgray}\textbf{0.68}$_{\color{BrickRed}{\text{+0.02}}}$ & 
    \cellcolor{lightgray}\textbf{0.42}$_{\color{BrickRed}{\text{+0.05}}}$
  & \cellcolor{lightgray}0.44$_{\color{OliveGreen}{\text{-0.02}}}$ & 
    \cellcolor{lightgray}0.68$_{\color{OliveGreen}{\text{-0.02}}}$ & 
    \cellcolor{lightgray}0.32$_{\color{OliveGreen}{\text{-0.02}}}$
  & \cellcolor{lightgray}0.42$_{\color{Gray}{\text{+0.00}}}$ & 
    \cellcolor{lightgray}\textbf{0.51}$_{\color{BrickRed}{\text{+0.01}}}$ & 
    \cellcolor{lightgray}0.36$_{\color{Gray}{\text{+0.00}}}$ \\
\midrule
\multirow{10}{*}{NPO} & Dense (0\%) & --
  & 0.64 & 0.64 & 0.63
  & 0.50 & 0.50 & 0.51
  & 0.45 & 0.41 & 0.50 \\
\cmidrule(l){2-12}
& \multirow{2}{*}{Mag.\,(50\%)} & w/o
  & 0.01 & 0.97 & 0.01
  & 0.45 & 0.34 & 0.68
  & 0.02 & 0.96 & 0.01 \\
& & \cellcolor{lightgray}w/
  & \cellcolor{lightgray}0.01$_{\color{Gray}{\text{+0.00}}}$ & 
    \cellcolor{lightgray}0.96$_{\color{OliveGreen}{\text{-0.01}}}$ & 
    \cellcolor{lightgray}0.01$_{\color{Gray}{\text{+0.00}}}$
  & \cellcolor{lightgray}0.42$_{\color{OliveGreen}{\text{-0.03}}}$ & 
    \cellcolor{lightgray}\textbf{0.38}$_{\color{BrickRed}{\text{+0.04}}}$ & 
    \cellcolor{lightgray}0.47$_{\color{OliveGreen}{\text{-0.21}}}$
  & \cellcolor{lightgray}0.01$_{\color{OliveGreen}{\text{-0.01}}}$ & 
    \cellcolor{lightgray}0.96$_{\color{Gray}{\text{+0.00}}}$ & 
    \cellcolor{lightgray}0.01$_{\color{Gray}{\text{+0.00}}}$ \\
\cmidrule(l){2-12}
& \multirow{2}{*}{SparseGPT\,(50\%)} & w/o
  & 0.53 & 0.72 & 0.42
  & 0.48 & 0.56 & 0.41
  & 0.44 & 0.44 & 0.44 \\
& & \cellcolor{lightgray}w/
  & \cellcolor{lightgray}\textbf{0.58}$_{\color{BrickRed}{\text{+0.05}}}$ & 
    \cellcolor{lightgray}\textbf{0.74}$_{\color{BrickRed}{\text{+0.02}}}$ & 
    \cellcolor{lightgray}\textbf{0.46}$_{\color{BrickRed}{\text{+0.04}}}$
  & \cellcolor{lightgray}\textbf{0.50}$_{\color{BrickRed}{\text{+0.02}}}$ & 
    \cellcolor{lightgray}0.54$_{\color{OliveGreen}{\text{-0.02}}}$ & 
    \cellcolor{lightgray}\textbf{0.46}$_{\color{BrickRed}{\text{+0.05}}}$
  & \cellcolor{lightgray}0.43$_{\color{OliveGreen}{\text{-0.01}}}$ & 
    \cellcolor{lightgray}0.44$_{\color{Gray}{\text{+0.00}}}$ & 
    \cellcolor{lightgray}0.43$_{\color{OliveGreen}{\text{-0.01}}}$ \\
\cmidrule(l){2-12}
& \multirow{2}{*}{Wanda\,(50\%)} & w/o
  & 0.44 & 0.75 & 0.31
  & 0.43 & 0.72 & 0.30
  & 0.41 & 0.45 & 0.35 \\
& & \cellcolor{lightgray}w/
  & \cellcolor{lightgray}\textbf{0.50}$_{\color{BrickRed}{\text{+0.06}}}$ & 
    \cellcolor{lightgray}\textbf{0.76}$_{\color{BrickRed}{\text{+0.01}}}$ & 
    \cellcolor{lightgray}\textbf{0.35}$_{\color{BrickRed}{\text{+0.04}}}$
  & \cellcolor{lightgray}\textbf{0.46}$_{\color{BrickRed}{\text{+0.03}}}$ & 
    \cellcolor{lightgray}0.70$_{\color{OliveGreen}{\text{-0.02}}}$ & 
    \cellcolor{lightgray}\textbf{0.34}$_{\color{BrickRed}{\text{+0.04}}}$
  & \cellcolor{lightgray}\textbf{0.42}$_{\color{BrickRed}{\text{+0.01}}}$ & 
    \cellcolor{lightgray}\textbf{0.47}$_{\color{BrickRed}{\text{+0.02}}}$ & 
    \cellcolor{lightgray}\textbf{0.38}$_{\color{BrickRed}{\text{+0.03}}}$ \\
\midrule

\multirow{10}{*}{UNDIAL} & Dense (0\%) & --
  & 0.70 & 0.73 & 0.67
  & 0.69 & 0.76 & 0.64
  & 0.67 & 0.66 & 0.68 \\
\cmidrule(l){2-12}
& \multirow{2}{*}{Mag.\,(50\%)} & w/o
  & 0.01 & 0.96 & 0.01
  & 0.01 & 0.52 & 0.00
  & 0.01 & 0.97 & 0.01 \\
& & \cellcolor{lightgray}w/
  & \cellcolor{lightgray}0.01$_{\color{Gray}{\text{+0.00}}}$ & 
    \cellcolor{lightgray}0.96$_{\color{Gray}{\text{+0.00}}}$ & 
    \cellcolor{lightgray}0.01$_{\color{Gray}{\text{+0.00}}}$
  & \cellcolor{lightgray}\textbf{0.02}$_{\color{BrickRed}{\text{+0.01}}}$ & 
    \cellcolor{lightgray}\textbf{0.60}$_{\color{BrickRed}{\text{+0.08}}}$ & 
    \cellcolor{lightgray}\textbf{0.01}$_{\color{BrickRed}{\text{+0.01}}}$
  & \cellcolor{lightgray}0.01$_{\color{Gray}{\text{+0.00}}}$ & 
    \cellcolor{lightgray}0.97$_{\color{Gray}{\text{+0.00}}}$ & 
    \cellcolor{lightgray}0.01$_{\color{Gray}{\text{+0.00}}}$ \\
\cmidrule(l){2-12}
& \multirow{2}{*}{SparseGPT\,(50\%)} & w/o
  & 0.55 & 0.72 & 0.45
  & 0.56 & 0.71 & 0.46
  & 0.55 & 0.71 & 0.45 \\
& & \cellcolor{lightgray}w/
  & \cellcolor{lightgray}\textbf{0.59}$_{\color{BrickRed}{\text{+0.04}}}$ & 
    \cellcolor{lightgray}\textbf{0.74}$_{\color{BrickRed}{\text{+0.02}}}$ & 
    \cellcolor{lightgray}\textbf{0.48}$_{\color{BrickRed}{\text{+0.03}}}$
  & \cellcolor{lightgray}\textbf{0.58}$_{\color{BrickRed}{\text{+0.02}}}$ & 
    \cellcolor{lightgray}\textbf{0.72}$_{\color{BrickRed}{\text{+0.01}}}$ & 
    \cellcolor{lightgray}\textbf{0.48}$_{\color{BrickRed}{\text{+0.02}}}$
  & \cellcolor{lightgray}\textbf{0.59}$_{\color{BrickRed}{\text{+0.04}}}$ & 
    \cellcolor{lightgray}0.70$_{\color{OliveGreen}{\text{-0.01}}}$ & 
    \cellcolor{lightgray}\textbf{0.50}$_{\color{BrickRed}{\text{+0.05}}}$ \\
\cmidrule(l){2-12}
& \multirow{2}{*}{Wanda\,(50\%)} & w/o
  & 0.50 & 0.72 & 0.38
  & 0.51 & 0.74 & 0.39
  & 0.47 & 0.71 & 0.35 \\
& & \cellcolor{lightgray}w/
  & \cellcolor{lightgray}\textbf{0.57}$_{\color{BrickRed}{\text{+0.07}}}$ & 
    \cellcolor{lightgray}\textbf{0.73}$_{\color{BrickRed}{\text{+0.01}}}$ & 
    \cellcolor{lightgray}\textbf{0.45}$_{\color{BrickRed}{\text{+0.07}}}$
  & \cellcolor{lightgray}\textbf{0.56}$_{\color{BrickRed}{\text{+0.05}}}$ & 
    \cellcolor{lightgray}0.72$_{\color{OliveGreen}{\text{-0.02}}}$ & 
    \cellcolor{lightgray}\textbf{0.46}$_{\color{BrickRed}{\text{+0.07}}}$
  & \cellcolor{lightgray}\textbf{0.56}$_{\color{BrickRed}{\text{+0.09}}}$ & 
    \cellcolor{lightgray}0.71$_{\color{Gray}{\text{+0.00}}}$ & 
    \cellcolor{lightgray}\textbf{0.45}$_{\color{BrickRed}{\text{+0.10}}}$ \\
\midrule

\multirow{10}{*}{SatImp} & Dense (0\%) & --
  & 0.55 & 0.91 & 0.39
  & 0.45 & 0.34 & 0.68
  & 0.41 & 0.31 & 0.61 \\
\cmidrule(l){2-12}
& \multirow{2}{*}{Mag.\,(50\%)} & w/o
  & 0.13 & 0.98 & 0.07
  & 0.02 & 0.96 & 0.01
  & 0.01 & 0.96 & 0.01 \\
& & \cellcolor{lightgray}w/
  & \cellcolor{lightgray}0.08$_{\color{OliveGreen}{\text{-0.05}}}$ & 
    \cellcolor{lightgray}0.98$_{\color{Gray}{\text{+0.00}}}$ & 
    \cellcolor{lightgray}0.04$_{\color{OliveGreen}{\text{-0.03}}}$
  & \cellcolor{lightgray}0.01$_{\color{OliveGreen}{\text{-0.01}}}$ & 
    \cellcolor{lightgray}0.96$_{\color{Gray}{\text{+0.00}}}$ & 
    \cellcolor{lightgray}0.01$_{\color{Gray}{\text{+0.00}}}$
  & \cellcolor{lightgray}0.01$_{\color{Gray}{\text{+0.00}}}$ & 
    \cellcolor{lightgray}0.96$_{\color{Gray}{\text{+0.00}}}$ & 
    \cellcolor{lightgray}0.01$_{\color{Gray}{\text{+0.00}}}$ \\
\cmidrule(l){2-12}
& \multirow{2}{*}{SparseGPT\,(50\%)} & w/o
  & 0.41 & 0.79 & 0.28
  & 0.44 & 0.45 & 0.44
  & 0.41 & 0.40 & 0.43 \\
& & \cellcolor{lightgray}w/
  & \cellcolor{lightgray}\textbf{0.51}$_{\color{BrickRed}{\text{+0.10}}}$ & 
    \cellcolor{lightgray}\textbf{0.82}$_{\color{BrickRed}{\text{+0.03}}}$ & 
    \cellcolor{lightgray}\textbf{0.37}$_{\color{BrickRed}{\text{+0.09}}}$
  & \cellcolor{lightgray}\textbf{0.47}$_{\color{BrickRed}{\text{+0.03}}}$ & 
    \cellcolor{lightgray}\textbf{0.50}$_{\color{BrickRed}{\text{+0.05}}}$ & 
    \cellcolor{lightgray}0.44$_{\color{Gray}{\text{+0.00}}}$
  & \cellcolor{lightgray}0.41$_{\color{Gray}{\text{+0.00}}}$ & 
    \cellcolor{lightgray}0.40$_{\color{Gray}{\text{+0.00}}}$ & 
    \cellcolor{lightgray}0.42$_{\color{OliveGreen}{\text{-0.01}}}$ \\
\cmidrule(l){2-12}
& \multirow{2}{*}{Wanda\,(50\%)} & w/o
  & 0.34 & 0.80 & 0.22
  & 0.44 & 0.70 & 0.32
  & 0.40 & 0.45 & 0.35 \\
& & \cellcolor{lightgray}w/
  & \cellcolor{lightgray}\textbf{0.47}$_{\color{BrickRed}{\text{+0.13}}}$ & 
    \cellcolor{lightgray}\textbf{0.83}$_{\color{BrickRed}{\text{+0.03}}}$ & 
    \cellcolor{lightgray}\textbf{0.33}$_{\color{BrickRed}{\text{+0.11}}}$
  & \cellcolor{lightgray}\textbf{0.45}$_{\color{BrickRed}{\text{+0.01}}}$ & 
    \cellcolor{lightgray}0.68$_{\color{OliveGreen}{\text{-0.02}}}$ & 
    \cellcolor{lightgray}\textbf{0.34}$_{\color{BrickRed}{\text{+0.02}}}$
  & \cellcolor{lightgray}0.39$_{\color{OliveGreen}{\text{-0.01}}}$ & 
    \cellcolor{lightgray}\textbf{0.46}$_{\color{BrickRed}{\text{+0.01}}}$ & 
    \cellcolor{lightgray}0.34$_{\color{OliveGreen}{\text{-0.01}}}$ \\
 
\bottomrule
\end{tabular}%
}
\label{tab:llama-3.2-1b-tofu-all}
\end{table*}

\noindent\textbf{Effect of Top-k Ratio for Gradient Masking.}
We evaluate the impact of the top-$k$ ratio hyperparameter on SAU's performance. As shown in Figure~\ref{fig:topk_ratio}, we vary the top-$k$ ratio across 0.1, 0.3, and 0.5 on TOFU Forget-10\% with Llama-3.1-8B at 50\% Magnitude pruning. The results demonstrate that top-$k=0.3$ consistently achieves the best performance across all four baseline methods, with an average aggregate score of 0.54 compared to 0.51 and 0.52 for top-$k=0.1$ and top-$k=0.5$ respectively. A ratio of 0.1 proves overly conservative, restricting updates to too few parameters, while 0.5 dilutes the unlearning signal across too many weights. The consistent trend validates that top-$k=0.3$ strikes an optimal balance between focusing on critical weights and maintaining sufficient parameter budget.

\begin{figure}[t]
\centering
   \begin{subfigure}[b]{0.211\textwidth}
       \includegraphics[width=\textwidth]{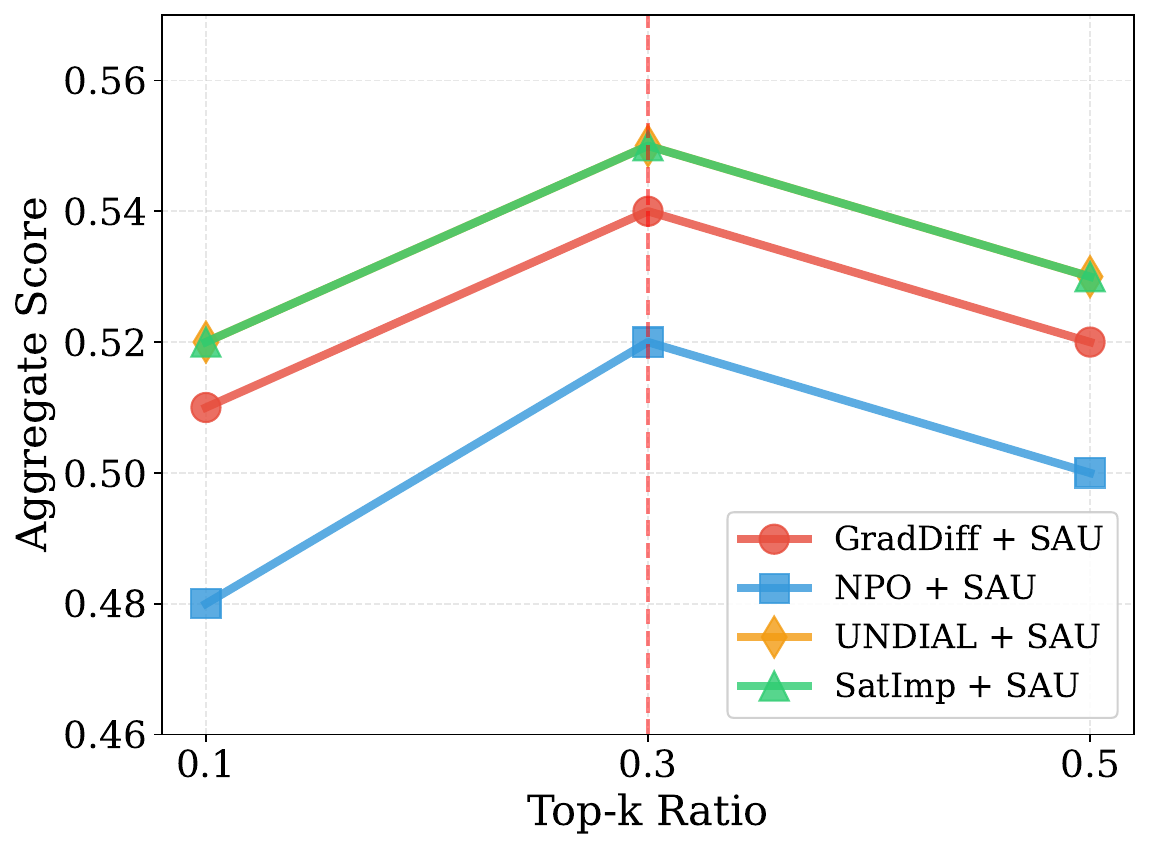}
       \caption{Effect of Top-k Ratio.}
       \label{fig:topk_ratio}
   \end{subfigure}
   \begin{subfigure}[b]{0.265\textwidth}
       \includegraphics[width=\textwidth]{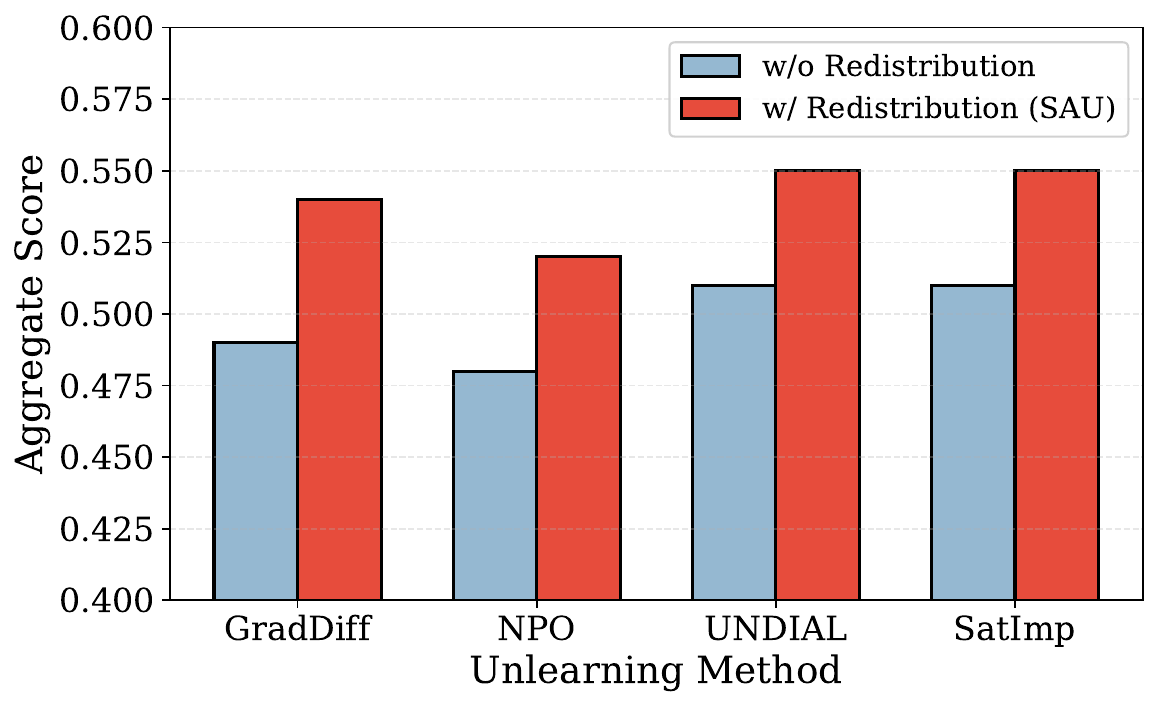}
       \caption{Effect of Redistribution.}
       \label{fig:retribution_ablation}
    \end{subfigure}

    \caption{Ablation studies on (a) top-k ratio and (b) importance-aware redistribution effect on TOFU Forget-10\% with Llama-3.1-8B at 50\% sparsity.}
\end{figure}

\noindent\textbf{Effect of Redistribution.}
We investigate the contribution of importance-aware weight redistribution by comparing SAU with a variant that only applies gradient masking. As shown in Figure~\ref{fig:retribution_ablation}, we evaluate both configurations on TOFU Forget-10\% with Llama-3.1-8B at 50\% Magnitude pruning. The full SAU with redistribution consistently outperforms the masking-only variant across all baseline methods, achieving an average aggregate score improvement of 8.5\%. Most notably, the redistribution mechanism substantially enhances utility preservation: GradDiff with redistribution achieves 0.62 utility compared to 0.47 without, representing a 32\% improvement. Similar patterns emerge across NPO, UNDIAL, and SatImp. This validates that importance-aware redistribution effectively compensates for lost unlearning capacity by amplifying gradient signals on surviving critical parameters.

\begin{figure}[htpb]
\centering
   \begin{subfigure}[b]{0.24\textwidth}
       \includegraphics[width=\textwidth]{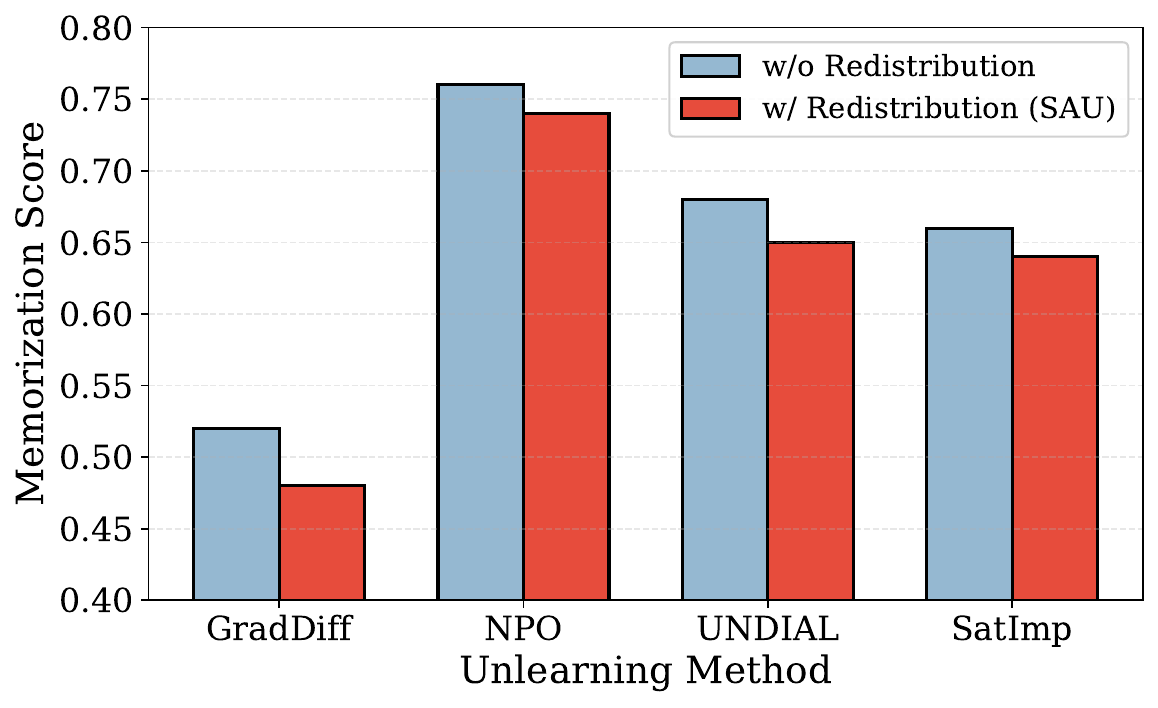}
       \caption{}
       \label{fig5:topk_ratio_memorization}
   \end{subfigure}
   \begin{subfigure}[b]{0.24\textwidth}
       \includegraphics[width=\textwidth]{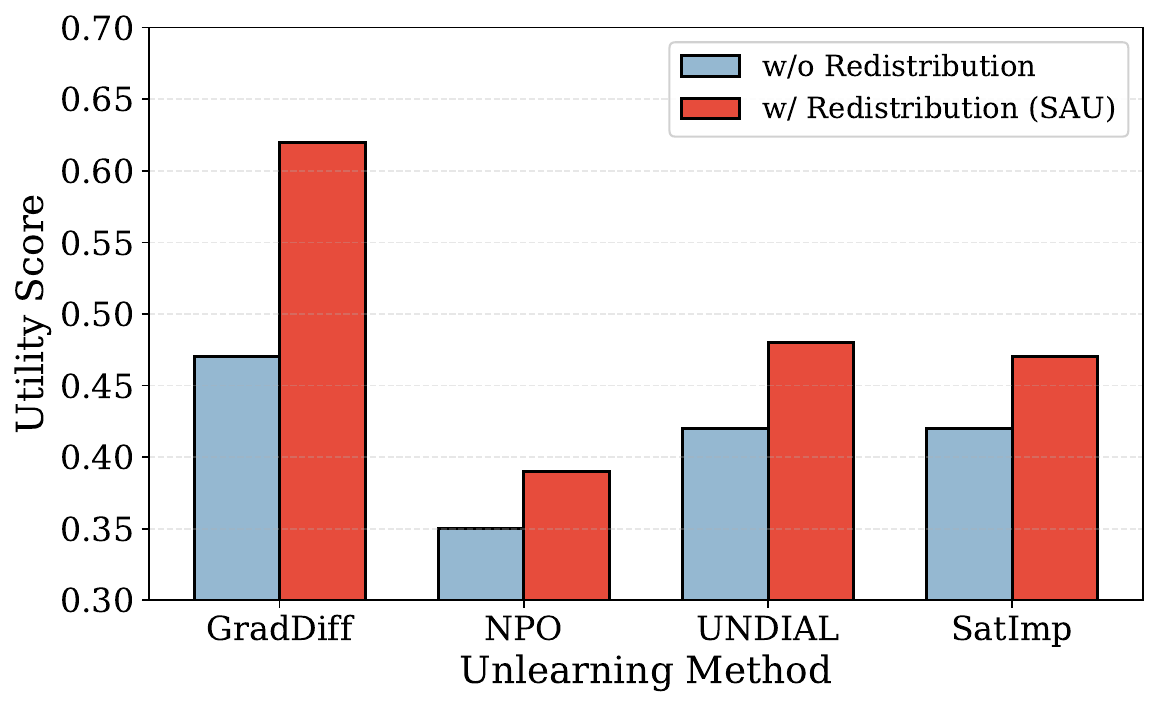}
       \caption{}
       \label{fig5:topk_ratio_utility}
    \end{subfigure}

    \caption{Ablation study on importance-aware weight redistribution. (a) Memorization scores: SAU with redistribution achieves lower memorization (better forgetting) across all methods. (b) Utility scores: redistribution substantially improves utility preservation, with an average 32\% improvement.}
\end{figure}

The results demonstrate that importance-aware weight redistribution provides substantial benefits across both forgetting and retention metrics. For memorization (Figure~\ref{fig5:topk_ratio_memorization}), SAU with redistribution consistently achieves lower scores compared to gradient masking alone, indicating more effective forgetting. Notably, GradDiff improves from 0.52 to 0.48, NPO from 0.76 to 0.74, UNDIAL from 0.68 to 0.65, and SatImp from 0.66 to 0.64.

More significantly, redistribution substantially enhances utility preservation, as shown in Figure~\ref{fig5:topk_ratio_utility}. GradDiff with redistribution achieves 0.62 utility compared to 0.47 without, representing a 32\% improvement. Similar patterns emerge across all methods: NPO improves from 0.35 to 0.39 with an 11\% gain, UNDIAL increases from 0.42 to 0.48 with a 14\% improvement, and SatImp rises from 0.42 to 0.47 with a 12\% boost. This validates that importance-aware redistribution effectively compensates for lost unlearning capacity by amplifying gradient signals on surviving critical parameters, achieving better forgetting while simultaneously improving retention rather than trading one for the other.

\section{Conclusion}

In this paper, we investigate the degradation of machine unlearning effectiveness on sparsified LLMs. Our analysis reveals that existing unlearning methods require updating all parameters, yet sparsification prunes substantial weights to zero, limiting forgetting capacity. To address this, we propose Sparsity-Aware Unlearning (SAU), which employs gradient masking to redirect updates to surviving weights and importance-aware redistribution to compensate for pruned parameters. Experiments demonstrate that SAU significantly outperforms existing methods on sparse LLMs, achieving effective forgetting while preserving utility. Our work bridges machine unlearning and model compression, enabling privacy-preserving LLM deployment under resource constraints.


\begin{thebibliography}{10}
\providecommand{\url}[1]{#1}
\csname url@samestyle\endcsname
\providecommand{\newblock}{\relax}
\providecommand{\bibinfo}[2]{#2}
\providecommand{\BIBentrySTDinterwordspacing}{\spaceskip=0pt\relax}
\providecommand{\BIBentryALTinterwordstretchfactor}{4}
\providecommand{\BIBentryALTinterwordspacing}{\spaceskip=\fontdimen2\font plus
\BIBentryALTinterwordstretchfactor\fontdimen3\font minus \fontdimen4\font\relax}
\providecommand{\BIBforeignlanguage}[2]{{%
\expandafter\ifx\csname l@#1\endcsname\relax
\typeout{** WARNING: IEEEtran.bst: No hyphenation pattern has been}%
\typeout{** loaded for the language `#1'. Using the pattern for}%
\typeout{** the default language instead.}%
\else
\language=\csname l@#1\endcsname
\fi
#2}}
\providecommand{\BIBdecl}{\relax}
\BIBdecl

\bibitem{grattafiori2024llama}
A.~Grattafiori, A.~Dubey, A.~Jauhri, A.~Pandey, A.~Kadian, A.~Al-Dahle, A.~Letman, A.~Mathur, A.~Schelten, A.~Vaughan \emph{et~al.}, ``The {Llama} 3 herd of models,'' \emph{arXiv preprint arXiv:2407.21783}, 2024.

\bibitem{achiam2023gpt}
J.~Achiam, S.~Adler, S.~Agarwal, L.~Ahmad, I.~Akkaya, F.~L. Aleman, D.~Almeida, J.~Altenschmidt, S.~Altman, S.~Anadkat \emph{et~al.}, ``Gpt-4 technical report,'' \emph{arXiv preprint arXiv:2303.08774}, 2023.

\bibitem{team2024gemma}
\BIBentryALTinterwordspacing
G.~Team, T.~Mesnard, C.~Hardin, R.~Dadashi, S.~Bhupatiraju, S.~Pathak, L.~Sifre, M.~Rivière, M.~S. Kale, J.~Love, P.~Tafti, L.~Hussenot, P.~G. Sessa, A.~Chowdhery, A.~Roberts, A.~Barua, A.~Botev, A.~Castro-Ros, A.~Slone, A.~Héliou, A.~Tacchetti, A.~Bulanova, A.~Paterson, B.~Tsai, B.~Shahriari, C.~L. Lan, C.~A. Choquette-Choo, C.~Crepy, D.~Cer, D.~Ippolito, D.~Reid, E.~Buchatskaya, E.~Ni, E.~Noland, G.~Yan, G.~Tucker, G.-C. Muraru, G.~Rozhdestvenskiy, H.~Michalewski, I.~Tenney, I.~Grishchenko, J.~Austin, J.~Keeling, J.~Labanowski, J.-B. Lespiau, J.~Stanway, J.~Brennan, J.~Chen, J.~Ferret, J.~Chiu, J.~Mao-Jones, K.~Lee, K.~Yu, K.~Millican, L.~L. Sjoesund, L.~Lee, L.~Dixon, M.~Reid, M.~Mikuła, M.~Wirth, M.~Sharman, N.~Chinaev, N.~Thain, O.~Bachem, O.~Chang, O.~Wahltinez, P.~Bailey, P.~Michel, P.~Yotov, R.~Chaabouni, R.~Comanescu, R.~Jana, R.~Anil, R.~McIlroy, R.~Liu, R.~Mullins, S.~L. Smith, S.~Borgeaud, S.~Girgin, S.~Douglas, S.~Pandya, S.~Shakeri, S.~De, T.~Klimenko, T.~Hennigan, V.~Feinberg, W.~Stokowiec,
  Y.~hui Chen, Z.~Ahmed, Z.~Gong, T.~Warkentin, L.~Peran, M.~Giang, C.~Farabet, O.~Vinyals, J.~Dean, K.~Kavukcuoglu, D.~Hassabis, Z.~Ghahramani, D.~Eck, J.~Barral, F.~Pereira, E.~Collins, A.~Joulin, N.~Fiedel, E.~Senter, A.~Andreev, and K.~Kenealy, ``Gemma: Open models based on gemini research and technology,'' 2024. [Online]. Available: \url{https://arxiv.org/abs/2403.08295}
\BIBentrySTDinterwordspacing

\bibitem{das2024security}
B.~C. Das, M.~H. Amini, and Y.~Wu, ``Security and privacy challenges of large language models: A survey,'' \emph{arXiv preprint arXiv:2402.00888}, 2024.

\bibitem{henderson2023foundation}
P.~Henderson, X.~Li, D.~Jurafsky, T.~Hashimoto, M.~A. Lemley, and P.~Liang, ``Foundation models and fair use,'' \emph{Journal of Machine Learning Research}, vol.~24, no. 400, pp. 1--79, 2023.

\bibitem{nasr2025scalable}
M.~Nasr, J.~Rando, N.~Carlini, J.~Hayase, M.~Jagielski, A.~F. Cooper, D.~Ippolito, C.~A. Choquette-Choo, F.~Tram{\`e}r, and K.~Lee, ``Scalable extraction of training data from aligned, production language models,'' in \emph{The Thirteenth International Conference on Learning Representations}, 2025.

\bibitem{yao2024survey}
Y.~Yao, J.~Duan, K.~Xu, Y.~Cai, Z.~Sun, and Y.~Zhang, ``A survey on large language model (llm) security and privacy: The good, the bad, and the ugly,'' \emph{High-Confidence Computing}, vol.~4, no.~2, p. 100211, 2024.

\bibitem{voigt2017eu}
P.~Voigt and A.~Von~dem Bussche, ``The eu general data protection regulation (gdpr),'' \emph{A practical guide, 1st ed., Cham: Springer International Publishing}, vol.~10, no. 3152676, pp. 10--5555, 2017.

\bibitem{zaeem2020effect}
R.~N. Zaeem and K.~S. Barber, ``The effect of the gdpr on privacy policies: Recent progress and future promise,'' \emph{ACM Transactions on Management Information Systems (TMIS)}, vol.~12, no.~1, pp. 1--20, 2020.

\bibitem{hadi2023survey}
M.~U. Hadi, R.~Qureshi, A.~Shah, M.~Irfan, A.~Zafar, M.~B. Shaikh, N.~Akhtar, J.~Wu, S.~Mirjalili \emph{et~al.}, ``A survey on large language models: Applications, challenges, limitations, and practical usage,'' \emph{Authorea Preprints}, 2023.

\bibitem{zhang2023right}
D.~Zhang, P.~Finckenberg-Broman, T.~Hoang, S.~Pan, Z.~Xing, M.~Staples, and X.~Xu, ``Right to be forgotten in the era of large language models: Implications, challenges, and solutions,'' \emph{arXiv preprint arXiv:2307.03941}, 2023.

\bibitem{ma2023llm}
X.~Ma, G.~Fang, and X.~Wang, ``Llm-pruner: On the structural pruning of large language models,'' \emph{Advances in neural information processing systems}, vol.~36, pp. 21\,702--21\,720, 2023.

\bibitem{sreenivas2024llm}
S.~T. Sreenivas, S.~Muralidharan, R.~Joshi, M.~Chochowski, A.~S. Mahabaleshwarkar, G.~Shen, J.~Zeng, Z.~Chen, Y.~Suhara, S.~Diao \emph{et~al.}, ``Llm pruning and distillation in practice: The minitron approach,'' \emph{arXiv preprint arXiv:2408.11796}, 2024.

\bibitem{cao2015towards}
Y.~Cao and J.~Yang, ``Towards making systems forget with machine unlearning,'' in \emph{2015 IEEE symposium on security and privacy}.\hskip 1em plus 0.5em minus 0.4em\relax IEEE, 2015, pp. 463--480.

\bibitem{bourtoule2021machine}
L.~Bourtoule, V.~Chandrasekaran, C.~A. Choquette-Choo, H.~Jia, A.~Travers, B.~Zhang, D.~Lie, and N.~Papernot, ``Machine unlearning,'' in \emph{2021 IEEE symposium on security and privacy (SP)}.\hskip 1em plus 0.5em minus 0.4em\relax IEEE, 2021, pp. 141--159.

\bibitem{graves2021amnesiac}
L.~Graves, V.~Nagisetty, and V.~Ganesh, ``Amnesiac machine learning,'' in \emph{Proceedings of the AAAI Conference on Artificial Intelligence}, vol.~35, no.~13, 2021, pp. 11\,516--11\,524.

\bibitem{thudi2022unrolling}
A.~Thudi, G.~Deza, V.~Chandrasekaran, and N.~Papernot, ``Unrolling sgd: Understanding factors influencing machine unlearning,'' in \emph{2022 IEEE 7th European Symposium on Security and Privacy (EuroS\&P)}.\hskip 1em plus 0.5em minus 0.4em\relax IEEE, 2022, pp. 303--319.

\bibitem{10607903}
H.~Xu, T.~Zhu, W.~Zhou, and W.~Zhao, ``Don't forget too much: Towards machine unlearning on feature level,'' \emph{IEEE Transactions on Dependable and Secure Computing}, vol.~22, no.~2, pp. 1313--1328, 2025.

\bibitem{koh2017understanding}
P.~W. Koh and P.~Liang, ``Understanding black-box predictions via influence functions,'' in \emph{International conference on machine learning}.\hskip 1em plus 0.5em minus 0.4em\relax PMLR, 2017, pp. 1885--1894.

\bibitem{izzo2021approximate}
Z.~Izzo, M.~A. Smart, K.~Chaudhuri, and J.~Zou, ``Approximate data deletion from machine learning models,'' in \emph{International Conference on Artificial Intelligence and Statistics}.\hskip 1em plus 0.5em minus 0.4em\relax PMLR, 2021, pp. 2008--2016.

\bibitem{11202435}
H.~Xu, T.~Zhu, L.~Zhang, and W.~Zhou, ``Really unlearned? verifying machine unlearning via influential sample pairs,'' \emph{IEEE Transactions on Dependable and Secure Computing}, vol.~23, no.~1, pp. 1671--1686, 2026.

\bibitem{liu2024large}
C.~Y. Liu, Y.~Wang, J.~Flanigan, and Y.~Liu, ``Large language model unlearning via embedding-corrupted prompts,'' \emph{arXiv preprint arXiv:2406.07933}, 2024.

\bibitem{gao2024practical}
C.~Gao, L.~Wang, C.~Weng, X.~Wang, and Q.~Zhu, ``Practical unlearning for large language models,'' \emph{arXiv preprint arXiv:2407.10223}, 2024.

\bibitem{yao2023large}
Y.~Yao, X.~Xu, and Y.~Liu, ``Large language model unlearning,'' in \emph{Socially Responsible Language Modelling Research}, 2023.

\bibitem{chen2023unlearn}
J.~Chen and D.~Yang, ``Unlearn what you want to forget: Efficient unlearning for {LLMs},'' in \emph{The 2023 Conference on Empirical Methods in Natural Language Processing}, 2023.

\bibitem{barbulescu2024textualsequenceownimproving}
\BIBentryALTinterwordspacing
G.-O. Barbulescu and P.~Triantafillou, ``To each (textual sequence) its own: Improving memorized-data unlearning in large language models,'' 2024. [Online]. Available: \url{https://arxiv.org/abs/2405.03097}
\BIBentrySTDinterwordspacing

\bibitem{liu2025rethinking}
S.~Liu, Y.~Yao, J.~Jia, S.~Casper, N.~Baracaldo, P.~Hase, Y.~Yao, C.~Y. Liu, X.~Xu, H.~Li \emph{et~al.}, ``Rethinking machine unlearning for large language models,'' \emph{Nature Machine Intelligence}, pp. 1--14, 2025.

\bibitem{wang2025rethinking}
Q.~Wang, J.~P. Zhou, Z.~Zhou, S.~Shin, B.~Han, and K.~Q. Weinberger, ``Rethinking {LLM} unlearning objectives: A gradient perspective and go beyond,'' in \emph{ICLR}, 2025.

\bibitem{liu2024towards}
Z.~Liu, G.~Dou, Z.~Tan, Y.~Tian, and M.~Jiang, ``Towards safer large language models through machine unlearning.''\hskip 1em plus 0.5em minus 0.4em\relax Association for Computational Linguistics, 2024.

\bibitem{tong2025robust}
Y.~Tong, Y.~Wang, J.~Yuan, and C.~Hu, ``Robust machine unlearning for quantized neural networks via adaptive gradient reweighting with similar labels,'' in \emph{Proceedings of the IEEE/CVF International Conference on Computer Vision}, 2025, pp. 20\,603--20\,612.

\bibitem{wang2024model}
W.~Wang, W.~Chen, Y.~Luo, Y.~Long, Z.~Lin, L.~Zhang, B.~Lin, D.~Cai, and X.~He, ``Model compression and efficient inference for large language models: A survey,'' \emph{arXiv preprint arXiv:2402.09748}, 2024.

\bibitem{zhou2024survey}
Z.~Zhou, X.~Ning, K.~Hong, T.~Fu, J.~Xu, S.~Li, Y.~Lou, L.~Wang, Z.~Yuan, X.~Li \emph{et~al.}, ``A survey on efficient inference for large language models,'' \emph{arXiv preprint arXiv:2404.14294}, 2024.

\bibitem{lee2020layer}
J.~Lee, S.~Park, S.~Mo, S.~Ahn, and J.~Shin, ``Layer-adaptive sparsity for the magnitude-based pruning,'' \emph{arXiv preprint arXiv:2010.07611}, 2020.

\bibitem{das2023beyond}
R.~J. Das, M.~Sun, L.~Ma, and Z.~Shen, ``Beyond size: How gradients shape pruning decisions in large language models,'' \emph{arXiv preprint arXiv:2311.04902}, 2023.

\bibitem{frantar2023sparsegpt}
E.~Frantar and D.~Alistarh, ``Sparsegpt: Massive language models can be accurately pruned in one-shot,'' in \emph{ICML}, 2023.

\bibitem{sun2024wanda}
M.~Sun \emph{et~al.}, ``A simple and effective pruning approach for large language models,'' in \emph{ICLR}, 2024.

\bibitem{Kirkpatrick_2017}
\BIBentryALTinterwordspacing
J.~Kirkpatrick, R.~Pascanu, N.~Rabinowitz, J.~Veness, G.~Desjardins, A.~A. Rusu, K.~Milan, J.~Quan, T.~Ramalho, A.~Grabska-Barwinska, D.~Hassabis, C.~Clopath, D.~Kumaran, and R.~Hadsell, ``Overcoming catastrophic forgetting in neural networks,'' \emph{Proceedings of the National Academy of Sciences}, vol. 114, no.~13, p. 3521–3526, Mar. 2017. [Online]. Available: \url{http://dx.doi.org/10.1073/pnas.1611835114}
\BIBentrySTDinterwordspacing

\bibitem{maini2024tofu}
P.~Maini, Z.~Feng, A.~Schwarzschild, Z.~C. Lipton, and J.~Z. Kolter, ``{TOFU}: A task of fictitious unlearning for {LLM}s,'' in \emph{COLM}, 2024.

\bibitem{shi2025muse}
W.~Shi, J.~Lee, Y.~Huang, S.~Malladi, J.~Zhao, A.~Holtzman, D.~Liu, L.~Zettlemoyer, N.~A. Smith, and C.~Zhang, ``{MUSE}: Machine unlearning six-way evaluation for language models,'' in \emph{ICLR}, 2025.

\bibitem{li2024wmdp}
N.~Li, A.~Pan, A.~Gopal, S.~Yue, D.~Berrios, A.~Gatti, J.~D. Li, A.-K. Dombrowski, S.~Goel, G.~Mukobi \emph{et~al.}, ``The {WMDP} benchmark: Measuring and reducing malicious use with unlearning,'' in \emph{ICML}, 2024.

\bibitem{shi2024detecting}
W.~Shi, A.~Ajith, M.~Xia, Y.~Huang, D.~Liu, T.~Blevins, D.~Chen, and L.~Zettlemoyer, ``Detecting pretraining data from large language models,'' in \emph{12th International Conference on Learning Representations, ICLR 2024}, 2024.

\bibitem{touvron2023llama2}
H.~Touvron, L.~Martin, K.~Stone, P.~Albert, A.~Almahairi, Y.~Babaei, N.~Bashlykov, S.~Batra, P.~Bhargava, S.~Bhosale \emph{et~al.}, ``Llama 2: Open foundation and fine-tuned chat models,'' \emph{arXiv preprint arXiv:2307.09288}, 2023.

\bibitem{tunstall2023zephyr}
L.~Tunstall, E.~Beeching, N.~Lambert, N.~Rajani, K.~Rasul, Y.~Belkada, S.~Huang, L.~von Werra, C.~Fourrier, N.~Habib \emph{et~al.}, ``Zephyr: Direct distillation of lm alignment,'' \emph{arXiv preprint arXiv:2310.16944}, 2023.

\bibitem{yangexploring}
P.~Yang, Q.~Wang, Z.~Huang, T.~Liu, C.~Zhang, and B.~Han, ``Exploring criteria of loss reweighting to enhance llm unlearning,'' in \emph{Forty-second International Conference on Machine Learning}, 2025.

\bibitem{zhang2024negative}
R.~Zhang, L.~Lin, Y.~Bai, and S.~Mei, ``Negative preference optimization: From catastrophic collapse to effective unlearning,'' in \emph{COLM}, 2024.

\bibitem{dong2025undial}
Y.~R. Dong, H.~Lin, M.~Belkin, R.~Huerta, and I.~Vuli{\'c}, ``{UNDIAL}: Self-distillation with adjusted logits for robust unlearning in large language models,'' in \emph{NAACL}, 2025.

\bibitem{openunlearning2025}
V.~Dorna, A.~Mekala, W.~Zhao, A.~McCallum, Z.~C. Lipton, J.~Z. Kolter, and P.~Maini, ``{OpenUnlearning}: Accelerating {LLM} unlearning via unified benchmarking of methods and metrics,'' in \emph{NeurIPS}, 2025.

\bibitem{hendrycks2020measuring}
D.~Hendrycks, C.~Burns, S.~Basart, A.~Zou, M.~Mazeika, D.~Song, and J.~Steinhardt, ``Measuring massive multitask language understanding,'' \emph{arXiv preprint arXiv:2009.03300}, 2020.

\bibitem{loshchilov2017fixing}
I.~Loshchilov and F.~Hutter, ``Fixing weight decay regularization in adam,'' \emph{CoRR}, 2017.

\end{thebibliography}


\begin{IEEEbiography}
[{\includegraphics[width=1in,height=1.24in,clip,keepaspectratio]{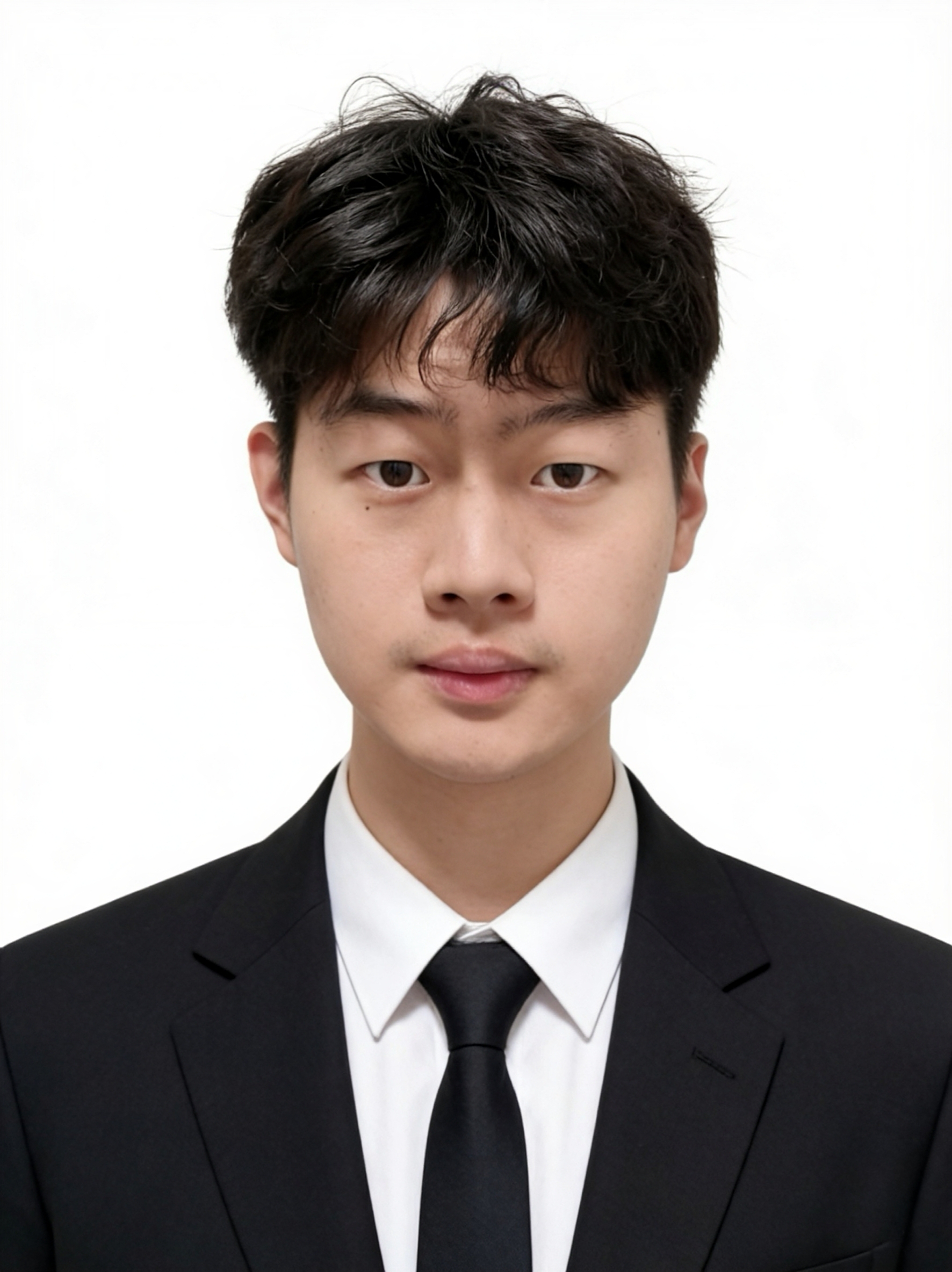}}]{Yuze Wang} received his B.S. degree from Wuhan University of Technology in 2024. He is currently working towards the M.S. degree in Computer Science and Technology with the computer science and artificial intelligence school, Wuhan University of Technology. His research interests include efficient LLMs , machine learning and trustworthy AI.
\end{IEEEbiography}

\begin{IEEEbiography}
[{\includegraphics[width=1in,height=1.24in,clip,keepaspectratio]{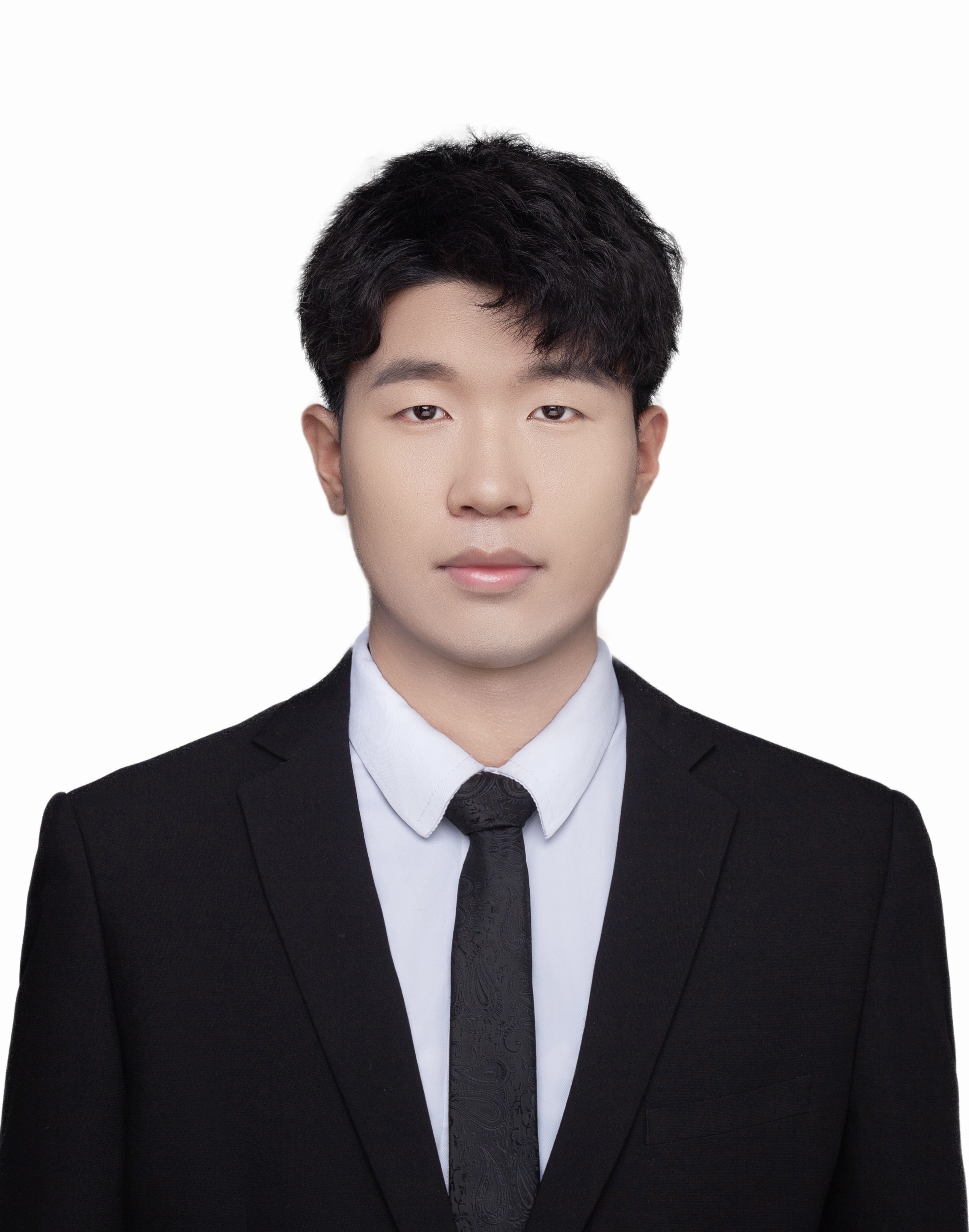}}]{Yujia Tong} received his B.S. degree from Wuhan University of Technology in 2024. He is currently working towards the M.S. degree in Computer Science and Technology with the computer science and
artificial intelligence school, Wuhan University of Technology. His research interests include Efficient Machine Learning and Trustworthy AI.
\end{IEEEbiography}

\begin{IEEEbiography}
[{\includegraphics[width=1in,height=1.24in,clip,keepaspectratio]{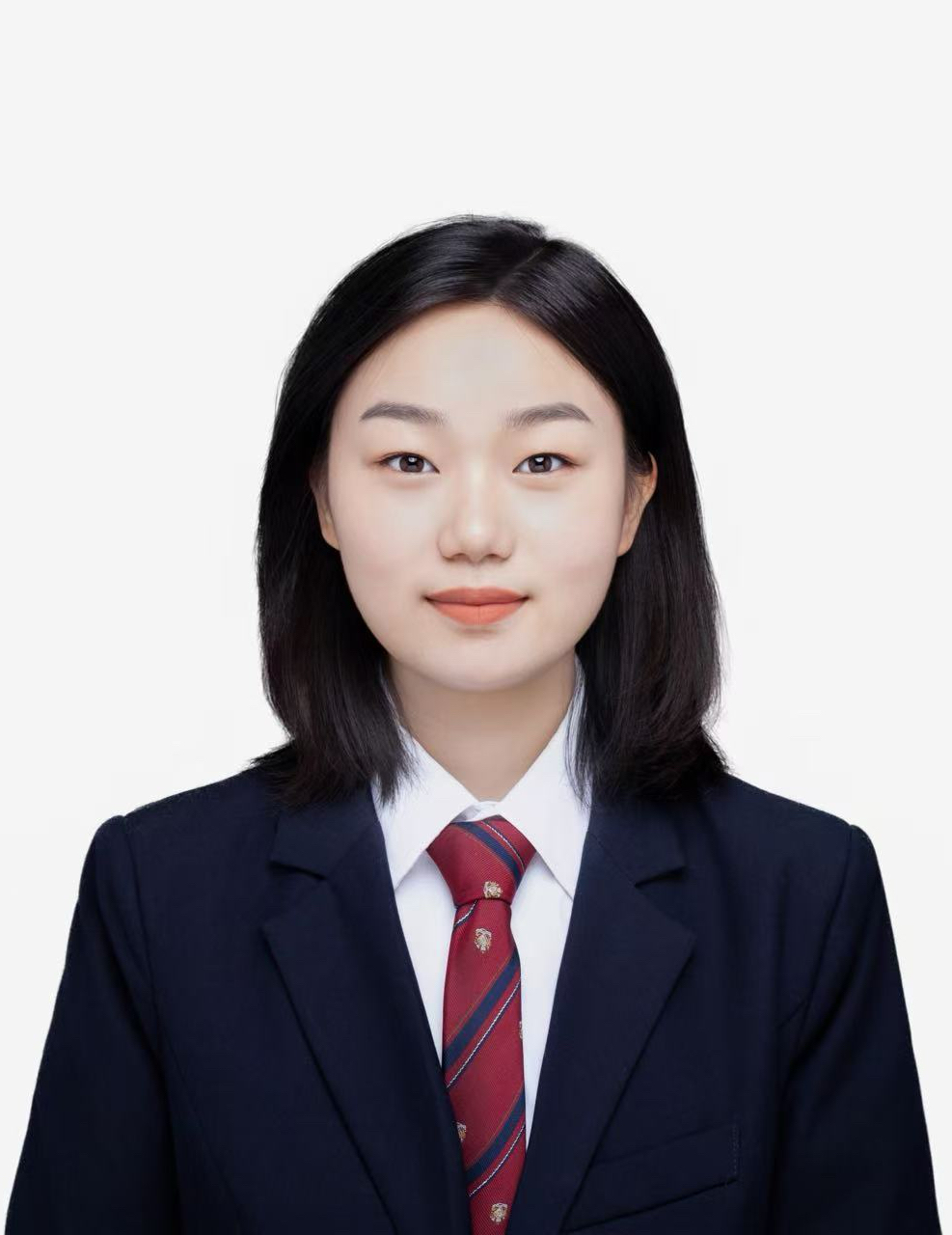}}]{Xuan Liu}
received her B.S. degree from The Hong Kong Polytechnic University in 2025. She is currently pursuing her Ph.D. degree at the University of California, San Diego. Her research interests include human-centered AI, human-computer interaction, and trustworthy AI.
\end{IEEEbiography}

\begin{IEEEbiography}
[{\includegraphics[width=1in,height=1.24in,clip,keepaspectratio]{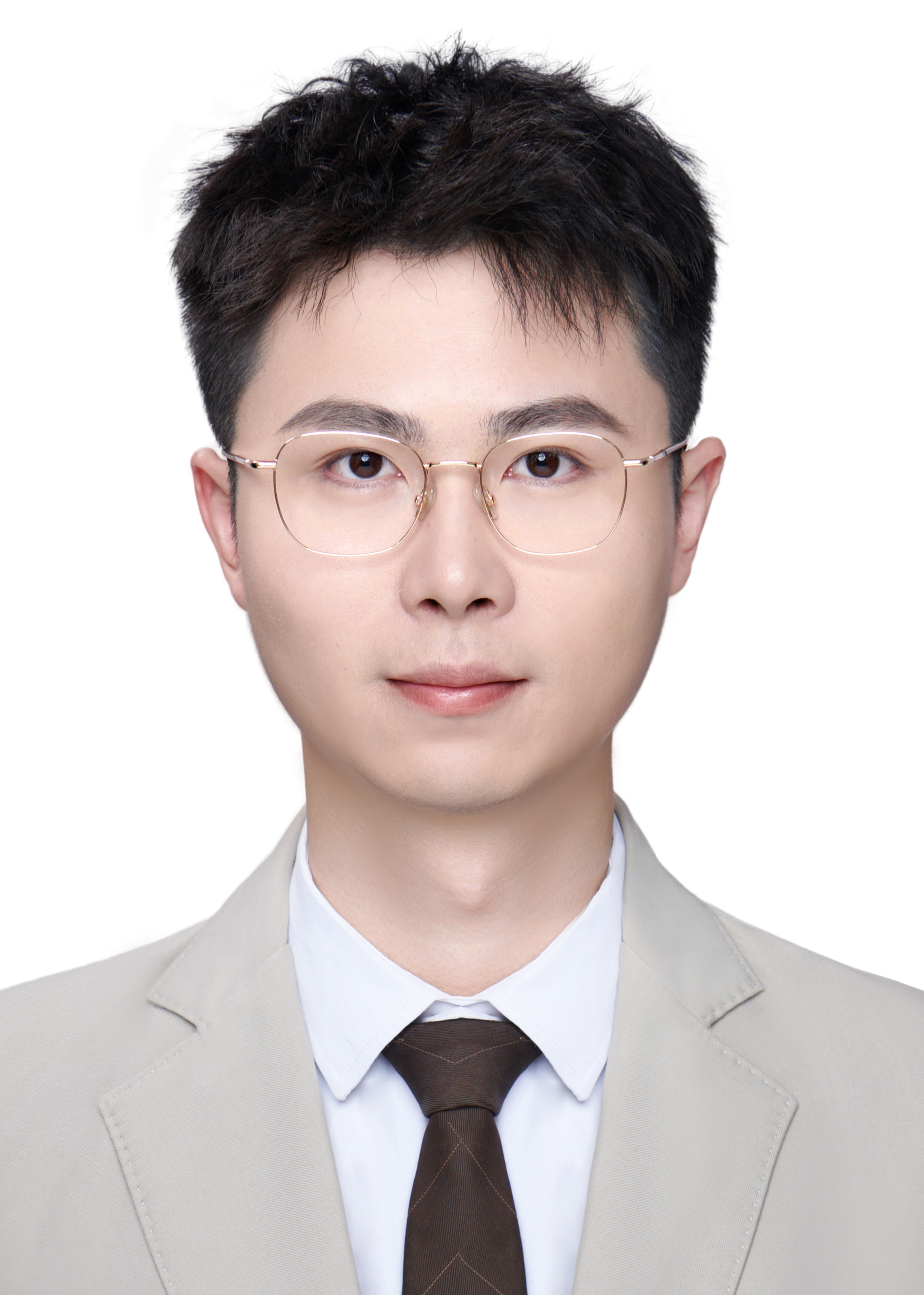}}]{Junhao Dong}
received the M.S. degree in computer science and technology from Sun Yat-sen University, Guangzhou, China, in 2023. He is currently working toward the Ph.D. degree with the College of Computing and Data Science, Nanyang Technological University, Singapore. His research interests include trustworthy AI, computer vision, and machine learning. He serves as an assistant program chair for NeurIPS 2025 and a session chair for KDD 2025.
\end{IEEEbiography}

\end{document}